\documentclass{article}
\usepackage[preprint]{style/neurips_2026}
\usepackage[T1]{fontenc}
\usepackage[utf8]{inputenc}
\usepackage{url}
\usepackage{microtype}
\usepackage{xcolor}
\usepackage{amsmath,amssymb,amsthm,mathtools}
\usepackage{amsfonts}
\usepackage{algorithm}
\usepackage{algpseudocode}
\usepackage{graphicx}
\usepackage{subcaption}
\usepackage{booktabs}
\usepackage{hyperref}
\usepackage{float}
\IfFileExists{placeins.sty}{\usepackage{placeins}}{\newcommand{\FloatBarrier}{}}

\graphicspath{{fig/}}

\newtheorem{theorem}{Theorem}
\newtheorem{lemma}{Lemma}
\newtheorem{definition}{Definition}
\newtheorem{proposition}{Proposition}
\newtheorem{corollary}{Corollary}
\newtheorem{assumption}{Assumption}
\theoremstyle{definition}
\newtheorem{remark}{Remark}

\newcommand{\E}{\mathbb{E}}
\newcommand{\Prb}{\mathbb{P}}
\newcommand{\R}{\mathbb{R}}
\newcommand{\cA}{\mathcal{A}}
\newcommand{\cE}{\mathcal{E}}
\newcommand{\cF}{\mathcal{F}}
\newcommand{\cH}{\mathcal{H}}
\newcommand{\cC}{\mathcal{C}}
\newcommand{\norm}[1]{\left\lVert #1 \right\rVert}
\newcommand{\ip}[2]{\left\langle #1, #2 \right\rangle}
\newcommand{\argmax}{\operatorname*{arg\,max}}

\algrenewcommand\algorithmicrequire{\textbf{Input:}}

\title{Direction-Aware Offline-to-Online Learning in Linear Contextual Bandits}
\author{
\mdseries
Zean Han$^{a}$, Ruihan Lin$^{b}$,\\
Zezhen Ding$^{c}$, Jiheng Zhang$^{d}$\\
Department of Industrial Engineering and Decision Analytics\\
The Hong Kong University of Science and Technology\\
\texttt{$^{a}$zhanax@connect.ust.hk, $^{b}$rlinah@connect.ust.hk}\\
\texttt{$^{c}$zdingah@connect.ust.hk, $^{d}$jiheng@ust.hk}
}
\date{}

\IfFileExists{style/experiment_macros.tex}{\input{style/experiment_macros.tex}}{}
\IfFileExists{style/paper_macros.tex}{

}{}

\begin{document}
\maketitle

\begin{abstract}
Many bandit systems are deployed with offline historical data, such as past logs from earlier policies.
Using these data can reduce early online exploration when they remain informative for the online problem.
When the offline and online environments differ, such data can be biased for the online problem.
For linear (contextual) bandits, this bias is directional: offline data may be informative in some feature directions and misleading in others.
However, prior work typically controls this gap through a known Euclidean bound on the model parameters, which we prove is too coarse: even with the offline parameter known, bias in a single unknown direction can force dimension-dependent regret.
To address this challenge, we introduce a directional bias certificate $(M_{\mathrm{bias}},\rho)$ that measures the offline-to-online gap through an $M_{\mathrm{bias}}$-induced norm and assigns different bias budgets to different directions.
Building on this certificate, we propose \emph{Ellipsoidal-MINUCB}, which augments the online learning with an offline-pooled branch that safely exploits historical data.
When the certificate is known, we show that the algorithm matches the standard SupLinUCB rate in the worst case and improves when offline coverage aligns with low-bias directions.
When the certificate is unknown, we estimate it adaptively from offline and accumulated online data and establish a corresponding regret guarantee.
Numerical experiments support the theory and show gains in aligned regimes.
\end{abstract}

\section{Introduction}
\label{sec:intro}

Many bandit systems are deployed with historical data already in hand, such as past logs or data from related environments. A natural goal is to use these data to reduce early online exploration. If the offline data remain informative in the live environment, they can give the learner a useful starting point and speed up online learning. This is the promise of offline-to-online learning \citep{zhang2019,tennenholtz2020,vijayan2025oope}.

The difficulty is that the offline and online environments need not match. The offline data may have
been collected under a different policy, population, or feature distribution, so the offline regression
parameter $\theta^\dagger$ need not equal the online reward parameter $\theta_*$, and the offline
estimate of rewards can be biased for the online problem. In linear contextual bandits, this mismatch
is directional.
Actions share a linear feature representation, so offline error is not equally important in all
directions. The offline estimate may be helpful along some directions and misleading along others, and
what matters is whether the bias lies in directions that affect online comparisons between actions.

Prior work often models offline-to-online mismatch through a single scalar bound, typically a Euclidean
bound $\|\theta_*-\theta^\dagger\|_2 \le \rho$ with a known radius $\rho$
\citep{cheung2024,zhou2025hybrid,he2025learninggap, zhou2026regret}. We show that this is fundamentally
too coarse. Such a bound controls only the total size of the mismatch; it does not say whether the
mismatch lies in directions that matter for online decisions. Our impossibility result
(Theorem~\ref{thm:lower}) makes this precise: even an oracle given the offline parameter $\theta^\dagger$
in advance must suffer worst-case regret of order $\Omega(\min\{\rho T,\sqrt{dT}\})$, because the same
scalar bound can describe an easy instance, where offline information already identifies the best
action, and a hard instance, where substantial online exploration is still necessary. This raises
the following question:
\begin{center}
\parbox{0.92\linewidth}{\centering\emph{How should one measure offline-to-online bias in a way that captures the
directions that matter for online decisions, and how can such a measure guide algorithm design?}}
\end{center}

We answer this question through four contributions:

\paragraph{Directional bias certificate:} We introduce a directional way to control offline-to-online mismatch:
\[
\|\theta_*-\theta^\dagger\|_{M_{\mathrm{bias}}}\le \rho,
\]
where $M_{\mathrm{bias}}\succ 0$ determines how much mismatch is allowed across directions and $\rho \ge 0$ is the corresponding radius. This pair turns qualitative information about which directions are stable and which are unreliable into quantitative, direction-wise bias budgets. We call such a pair $(M_{\mathrm{bias}},\rho)$ a \emph{directional bias certificate}, since it certifies how much offline-to-online bias is allowed along each direction. The intuition behind the certificate is given in Appendix~\ref{app:directional-bias-bound}.

\paragraph{Algorithm design:} Using this certificate, we design \emph{Ellipsoidal-MINUCB}, which scores each action by the minimum of two upper confidence bounds: a purely online one and an offline-informed one built from the certificate. We use the offline-informed bound only when it is tighter, so misleading offline information cannot remove the usual online fallback. We prove regret bounds showing that the algorithm matches the standard SupLinUCB rate in the worst case, with the offline-informed term improving when offline coverage and directional bias are aligned with the directions that drive online decisions.

\paragraph{Adaptive estimation of the directional bias certificate:} We also study the case where the directional bias certificate is not known before deployment. We propose an adaptive construction of the certificate which updates epoch by epoch from the offline samples together with the online observations collected so far, and we establish a corresponding regret guarantee. This extends the framework to settings where directional bias information is not available in advance, which largely enhances the practicality of our methods.

\paragraph{Empirical validation:} We complement the theory with numerical simulations comparing
our algorithms against online, warm-started, and offline-only baselines. The results confirm the
predicted pattern: gains are largest when offline coverage aligns with low-bias directions, and the
algorithm falls back gracefully to the online baseline otherwise.

\section{Problem Setup}
\label{sec:setup}

We specify the online bandit model, offline regression data, and a directional bias certificate, then record an
offline confidence region for $\theta_*$ used in Section~\ref{sec:algorithm}.

Let $\cF_{t-1}$ be the $\sigma$-algebra generated by $(x_{s,\cdot},a_s,r_s)_{s<t}$ and the current
contexts $\{x_{t,a}\}_{a\in[K]}$. We consider a stochastic linear contextual bandit over horizon $T$.

\begin{assumption}[Linear reward with bounded features]
\label{asm:linear}
At each round $t\in[T]$, the learner observes a finite action set
$\{x_{t,a}\}_{a \in [K]} \subset \R^d$ with $\|x_{t,a}\|_2 \le 1$, selects $a_t$, and receives
$r_t = x_{t,a_t}^{\top}\theta_* + \eta_t \in [0,1]$, where $\theta_* \in \R^d$ is unknown with
$\|\theta_*\|_2 \le S$. The action set at round $t$ is revealed before the reward is observed, so
$x_{t,a_t}$ and the SupLinUCB layer assignment (Section~\ref{sec:algorithm}) are
$\cF_{t-1}$-measurable.
\end{assumption}

\begin{assumption}[Conditionally sub-Gaussian online noise]
\label{asm:subg}
$\E[\eta_t \mid \cF_{t-1}] = 0$ and $\eta_t$ is conditionally $\sigma$-sub-Gaussian given
$\cF_{t-1}$.
\end{assumption}

The cumulative regret is
\[
R_T := \sum_{t=1}^T \left( \max_{a \in [K]} x_{t,a}^{\top}\theta_* - x_{t,a_t}^{\top}\theta_* \right).
\]

Before the online phase, the learner observes  offline regression samples
$D_{\mathrm{off}} = \{(z_i,y_i)\}_{i=1}^{n_{\mathrm{off}}}$. The difficulty is that the offline
parameter $\theta^\dagger$ need not equal the online parameter $\theta_*$.

\begin{assumption}[Offline regression model]
\label{asm:offline}
$y_i = z_i^{\top}\theta^\dagger + \xi_i$ with $\|z_i\|_2\le 1$, $\|\theta^\dagger\|_2\le S$, and
$\E[\xi_i\mid z_i]=0$. The offline covariates $\{z_i\}$ are non-adaptive (fixed before the online
phase), and $\xi_i$ is conditionally $\sigma$-sub-Gaussian given $z_i$.
\end{assumption}

\begin{definition}[Directional bias certificate]
\label{def:directional-bias-bound}
Let $M_{\mathrm{bias}}\in\mathbb{S}_{++}^d$ and $\rho\ge 0$ be known.
The pair $(M_{\mathrm{bias}},\rho)$ is called a valid directional bias certificate for
$(\theta_*,\theta^\dagger)$ if
\[
\|\theta_*-\theta^\dagger\|_{M_{\mathrm{bias}}} \le \rho.
\]
Equivalently, every direction $x \in \R^d$ satisfies
\[
|x^{\top}(\theta_*-\theta^\dagger)| \le \rho \|x\|_{M_{\mathrm{bias}}^{-1}}.
\]
\end{definition}

\begin{remark}
\label{rem:directional-bias-interpretation}
In linear contextual bandits, rewards depend on actions only through features, so the natural
object to control is parameter error \emph{along the directions} that enter UCB/LCB scores.
A valid $(M_{\mathrm{bias}},\rho)$ is side information that encodes this anisotropy: it implies
per-direction bounds with penalties $\rho\|x\|_{M_{\mathrm{bias}}^{-1}}$ matching the linear
optimistic template, and it is not identified from bandit feedback alone without additional
structure. Intuitively, a large eigenvalue of $M_{\mathrm{bias}}$ means that mismatch is tightly
constrained along the corresponding eigenvector, while a small eigenvalue allows a larger
offline-to-online parameter difference in that direction. A detailed explanation is given in
Appendix~\ref{app:directional-bias-bound}.
By contrast, a Euclidean bound $\norm{\theta_*-\theta^\dagger}_2 \le \rho$ controls only the
\emph{magnitude} of the bias and can hide whether the bias lies along a regret-relevant
direction.
Appendix~\ref{app:isotropic-too-coarse} gives a finite-action Bernoulli specialization in which
three instances share the same offline parameter and offline data law and satisfy the same scalar
radius; in one case a fixed action is optimal throughout, while in a paired hard case the learner
must identify the sign of a hidden bias and every policy suffers regret of order
$\Omega(\rho^2T/\rho_0)$ on one of the two instances.
\end{remark}

To use this side information algorithmically, we combine it with the offline ridge
estimate. The offline samples yield a confidence ellipsoid around $\theta^\dagger$. Together with the
directional bias certificate, Proposition~\ref{prop:offline-region} fuses the two into a single confidence region for
$\theta_*$.

The offline samples induce the usual ridge objects
\[
G_{\mathrm{off}} := I + \sum_{i=1}^{n_{\mathrm{off}}} z_i z_i^{\top},
\qquad
\hat\theta_{\mathrm{off}} := G_{\mathrm{off}}^{-1} \sum_{i=1}^{n_{\mathrm{off}}} z_i y_i.
\]
Here $G_{\mathrm{off}}$ records offline coverage, while $\hat\theta_{\mathrm{off}}$ estimates the
offline parameter $\theta^\dagger$. The ridge concentration inequality controls
$\hat\theta_{\mathrm{off}}-\theta^\dagger$ in the $G_{\mathrm{off}}$ metric. Then the directional bias certificate
lifts this offline ellipsoid to a confidence region for $\theta_*$.

\begin{proposition}[Offline-to-online confidence region]
\label{prop:offline-region}
For
\[
\beta_{\mathrm{off}}(\delta) :=
\sigma \sqrt{2 \log \frac{\det(G_{\mathrm{off}})^{1/2}}{\delta}}
+ S,
\]
with probability at least $1-\delta$,
\[
\theta_* \in
\left\{
\hat\theta_{\mathrm{off}} + u + v :
\|u\|_{G_{\mathrm{off}}} \le \beta_{\mathrm{off}}(\delta),
\ \|v\|_{M_{\mathrm{bias}}} \le \rho
\right\}
=:
\cC_{\mathrm{off}}.
\]
Consequently, the support function of this region is
\[
\sup_{\theta \in \cC_{\mathrm{off}}} x^{\top}\theta
=
x^{\top}\hat\theta_{\mathrm{off}}
+ \beta_{\mathrm{off}}(\delta)\|x\|_{G_{\mathrm{off}}^{-1}}
+ \rho\|x\|_{M_{\mathrm{bias}}^{-1}}.
\]
\end{proposition}

\section{Ellipsoidal-MINUCB}
\label{sec:algorithm}

So far, $\theta_*$ is described only through the offline-informed region above. Online play requires
layerwise scores that update with streaming data. The next paragraphs define those scores.
Fix an offline confidence budget $\delta_{\mathrm{off}}$ and write
$\beta_{\mathrm{off}}:=\beta_{\mathrm{off}}(\delta_{\mathrm{off}})$ throughout this section.

\paragraph{Online branch.}
For each layer $\ell$, define
\begin{equation}
\label{eq:online-objects}
A_{t,\ell}:=\sum_{s\in\cH_{t,\ell}}x_{s,a_s}x_{s,a_s}^{\top},
\qquad
V_{t,\ell} := I + A_{t,\ell},
\qquad
\hat\theta^{\mathrm{on}}_{t,\ell} := V_{t,\ell}^{-1} b_{t,\ell},
\qquad
b_{t,\ell} := \sum_{s \in \cH_{t,\ell}} x_{s,a_s} r_s.
\end{equation}
The online UCB is
\begin{equation}
\label{eq:online-ucb}
U^{\mathrm{on}}_{t,\ell}(a)
=
x_{t,a}^{\top}\hat\theta^{\mathrm{on}}_{t,\ell}
+ \mathrm{rad}^{\mathrm{on}}_{t,\ell}(x), 
\end{equation}
where $\mathrm{rad}^{\mathrm{on}}_{t,\ell}(x)
:=
\beta^{\mathrm{on}}_{t,\ell}\|x\|_{V_{t,\ell}^{-1}}$
with the corresponding online LCB $L^{\mathrm{on}}_{t,\ell}(a)$ defined symmetrically, and
\[
\beta^{\mathrm{on}}_{t,\ell}
:=
\sigma \sqrt{2 \log \frac{\det(V_{t,\ell})^{1/2}}{\delta_{t,\ell}^{\mathrm{on}}}}
+ S
\]
is the layerwise ridge confidence radius for budget $\delta_{t,\ell}^{\mathrm{on}}$.

\paragraph{Pooled branch.}
The offline-informed estimator is
\begin{equation}
\label{eq:pooled-estimator}
\hat\theta^{\mathrm{pool}}_{t,\ell}
:=
(A_{t,\ell}+G_{\mathrm{off}})^{-1}
\bigl(b_{t,\ell} + G_{\mathrm{off}}\hat\theta_{\mathrm{off}}\bigr).
\end{equation}
The confidence interval radius is
\begin{equation}
\mathrm{rad}^{\mathrm{pool}}_{t,\ell}(x)
:=
\bigl(\gamma_{t,\ell}+\beta_{\mathrm{off}}\bigr)\|x\|_{(A_{t,\ell}+G_{\mathrm{off}})^{-1}}
+ \rho \,\psi_{t,\ell}(x).
\label{eq:pooled-radius}
\end{equation}
where
\begin{equation}
\label{eq:bias-routing}
\psi_{t,\ell}(x)
:=
\left\|
M_{\mathrm{bias}}^{-1/2}
G_{\mathrm{off}}(A_{t,\ell}+G_{\mathrm{off}})^{-1}x
\right\|_2.
\end{equation}
The pooled martingale radius is
\begin{equation}
\label{eq:pooled-gamma}
\gamma_{t,\ell}
:=
\sigma \sqrt{
2 \log
\frac{
\det(A_{t,\ell}+G_{\mathrm{off}})^{1/2}
}{
\det(G_{\mathrm{off}})^{1/2}\delta_{t,\ell}^{\mathrm{pool}}
}
}.
\end{equation}
Although $\gamma_{t,\ell}$ appears in the pooled branch, it controls only the online martingale
noise measured in the pooled geometry $A_{t,\ell}+G_{\mathrm{off}}$; the offline ridge uncertainty is
carried by the separate $\beta_{\mathrm{off}}$ term, and no additional $+S$ regularization slack is
needed in this branch.
The pooled UCB is
\begin{equation}
\label{eq:pooled-ucb}
U^{\mathrm{pool}}_{t,\ell}(a)
=
x_{t,a}^{\top}\hat\theta^{\mathrm{pool}}_{t,\ell}
+ \mathrm{rad}^{\mathrm{pool}}_{t,\ell}(x_{t,a}),
\end{equation}
with the corresponding pooled LCB $L^{\mathrm{pool}}_{t,\ell}(a)$ defined symmetrically.

\begin{remark}
\label{rem:pooled-radius-interpretation}
Pooling replaces $\|x\|_{V_{t,\ell}^{-1}}$ by $\|x\|_{(A_{t,\ell}+G_{\mathrm{off}})^{-1}}$ when
offline data is informative along $x$.
The term $\rho\,\psi_{t,\ell}(x)$ depends on two factors: the pooling map
$G_{\mathrm{off}}(A_{t,\ell}+G_{\mathrm{off}})^{-1}$ and the bias geometry $M_{\mathrm{bias}}$.
Along~$x$, offline data helps only if pooling makes the radius tighter than the purely online radius when $\rho\,\psi_{t,\ell}(x)$ is included.
\end{remark}

\paragraph{Branch aggregation.}
The final aggregated UCB and LCB are
\begin{equation}
\label{eq:branch-aggregation}
U^{\min}_{t,\ell}(a)
:=
\min\{U^{\mathrm{on}}_{t,\ell}(a),U^{\mathrm{pool}}_{t,\ell}(a)\},
\qquad
L^{\max}_{t,\ell}(a)
:=
\max\{L^{\mathrm{on}}_{t,\ell}(a),L^{\mathrm{pool}}_{t,\ell}(a)\}.
\end{equation}
This design makes the safety mechanism explicit: the pooled branch can only help if it is tighter,
and the online branch remains available when the directional bias certificate is weak or the queried
direction is poorly covered offline.

Let
\begin{equation}
\label{eq:aggregated-width}
w_{t,\ell}(a) := U^{\min}_{t,\ell}(a)-L^{\max}_{t,\ell}(a)
\end{equation}
denote the aggregated interval width.
The full round procedure is given in Algorithm~\ref{alg:ellipsoidal-minucb}.

The regret analysis of Algorithm~\ref{alg:ellipsoidal-minucb} relies on two technical points.
First, online data are partitioned across layers in SupLinUCB, so the ridge estimator at each layer uses a random subsequence of online rewards rather than the full trajectory. We handle these adaptive subsequences by viewing each layer's history through predictable masks and applying a masked self-normalized martingale bound.
Second, the pooled estimator merges online Gram information $A_{t,\ell}$ with the offline design $G_{\mathrm{off}}$.
For each direction $x$, we control $x^\top\hat{\theta}^{\mathrm{pool}}_{t,\ell}-x^\top\theta_*$ by summarizing online and offline estimation errors around $\theta^\dagger$ with the elliptic norm $\|x\|_{(A_{t,\ell}+G_{\mathrm{off}})^{-1}}$ at radius scale $\gamma_{t,\ell}+\beta_{\mathrm{off}}$.
The offline-to-online bias $\theta^\dagger-\theta_*$ is treated differently: under the directional bias certificate it contributes $\rho\,\psi_{t,\ell}(x)$, so extra uncertainty from bias depends on the direction instead of being treated uniformly.

\begin{algorithm}[t]
\caption{Ellipsoidal-MINUCB}
\label{alg:ellipsoidal-minucb}
\begin{algorithmic}[1]
\Require Offline data $D_{\mathrm{off}}$, directional bias certificate $(M_{\mathrm{bias}},\rho)$, horizon $T$,
offline confidence budget $\delta_{\mathrm{off}}$,
online confidence schedule
$\{\delta_{t,\ell}^{\mathrm{on}}\}_{t \le T,\ell \le \lceil \log_2 T \rceil}$, and pooled
confidence schedule $\{\delta_{t,\ell}^{\mathrm{pool}}\}_{t \le T,\ell \le \lceil \log_2 T \rceil}$
\State Compute $G_{\mathrm{off}}$, $\hat\theta_{\mathrm{off}}$, and
$\beta_{\mathrm{off}}:=\beta_{\mathrm{off}}(\delta_{\mathrm{off}})$ from
$D_{\mathrm{off}}$
\State $L \gets \lceil \log_2 T \rceil$
\State Initialize $\cH_{1,\ell} \gets \varnothing$ for all $\ell = 0,1,\dots,L$
\For{$t = 1,2,\dots,T$}
    \State Set $\cH_{t+1,\ell} \gets \cH_{t,\ell}$ for every $\ell = 0,1,\dots,L$
    \State Observe the action set $\{x_{t,a}\}_{a \in [K]}$
    \State $\cA_{t,0} \gets [K]$
    \For{$\ell = 0,1,\dots,L$}
        \State Build $A_{t,\ell}$, $V_{t,\ell}$, and $b_{t,\ell}$ from \eqref{eq:online-objects}
        \State Compute online scores by \eqref{eq:online-ucb}
        \State Compute pooled scores by \eqref{eq:pooled-estimator}--\eqref{eq:pooled-ucb}
        \State Compute aggregated widths by \eqref{eq:branch-aggregation}--\eqref{eq:aggregated-width}
        \For{each $a \in \cA_{t,\ell}$}
            \State Compute aggregated UCB $U^{\min}_{t,\ell}(a)$, aggregated LCB $L^{\max}_{t,\ell}(a)$, and $w_{t,\ell}(a)$
        \EndFor
        \If{$\ell = L$}
            \State Play $a_t \in \argmax_{a \in \cA_{t,\ell}} U^{\min}_{t,\ell}(a)$
            \State Observe reward $r_t$ and append $(x_{t,a_t},r_t)$ to $\cH_{t+1,\ell}$
            \State Record $\ell_t \gets \ell$ and break the layer loop
        \ElsIf{$\max_{a \in \cA_{t,\ell}} w_{t,\ell}(a) > 2^{-\ell}$}
            \State Play
            $a_t \in \argmax_{a \in \cA_{t,\ell}: w_{t,\ell}(a) > 2^{-\ell}} U^{\min}_{t,\ell}(a)$
            \State Observe reward $r_t$ and append $(x_{t,a_t},r_t)$ to $\cH_{t+1,\ell}$
            \State Record $\ell_t \gets \ell$ and break the layer loop
        \Else
            \State $m_{t,\ell} \gets \max_{a \in \cA_{t,\ell}} L^{\max}_{t,\ell}(a)$
            \State
            $\cA_{t,\ell+1}
            \gets
            \{a \in \cA_{t,\ell} : U^{\min}_{t,\ell}(a) \ge m_{t,\ell}\}$
        \EndIf
    \EndFor
\EndFor
\end{algorithmic}
\end{algorithm}

\section{Main Results}
\label{sec:theory}

The next theorem states our main high-probability regret bound. Corollaries and structural comparisons
follow in the same section.

\begin{theorem}[High-probability guarantee]
\label{thm:main}
Write $L:=\lceil\log_2 T\rceil$. Run Ellipsoidal-MINUCB with confidence budgets
$\delta_{\mathrm{off}},\delta_{\mathrm{on}},
\{\delta_{t,\ell}^{\mathrm{on}}\},\{\delta_{t,\ell}^{\mathrm{pool}}\}$ satisfying
\[
\sum_{t=1}^T\sum_{\ell=0}^{L}\delta_{t,\ell}^{\mathrm{on}}
\le \delta_{\mathrm{on}},
\qquad
\delta_{\mathrm{off}}+\delta_{\mathrm{on}}
+\sum_{t=1}^T\sum_{\ell=0}^{L}\delta_{t,\ell}^{\mathrm{pool}}
\le T^{-2}.
\]
Then, with probability at least $1-T^{-2}$,
\[
R_T
\le
\min\Bigl\{
C_{\mathrm{SL}}\sqrt{Td\,L_T^3},\;
8\,\Gamma_{\mathrm{pool}}\sqrt{2T\,\widetilde\Lambda_T}
+ 8\rho\,\widetilde\Psi_T
\Bigr\},
\]
where $\widetilde\Lambda_T:=\sum_{\ell=0}^L\Lambda_\ell$ aggregates layerwise pooled log-determinant
increments $(\Lambda_\ell)$ built from $(Q_\ell,A_\ell^{\mathrm{fin}})$ defined in
Appendix~\ref{app:proof-main},
$\widetilde\Psi_T:=\sum_{t=1}^T \psi_{t,\ell_t}(x_{t,a_t})$,
$L_T:=\log(2KT\log T/\delta_{\mathrm{on}})$,
$\beta_{\mathrm{off}}:=\beta_{\mathrm{off}}(\delta_{\mathrm{off}})$,
$\delta_{\min}^{\mathrm{pool}}:=\min_{t,\ell}\delta_{t,\ell}^{\mathrm{pool}}$,
$\Lambda_{\max}:=\max_{0\le\ell\le L}\Lambda_\ell$, and
$\Gamma_{\mathrm{pool}}
:=\beta_{\mathrm{off}}+\sigma\sqrt{\Lambda_{\max}+2\log(1/\delta_{\min}^{\mathrm{pool}})}$.
\end{theorem}

Theorem~\ref{thm:main} reports whichever is smaller: the purely online regret, or the regret of the pooled branch.
For the pooled branch regret, $\widetilde\Lambda_T$ tracks how much information comes from the pooled branch, while $\widetilde\Psi_T$ tracks how much we must widen confidence intervals because offline estimates may be biased along directions that appear online.
The prefactor $\Gamma_{\mathrm{pool}}$ uses the largest pooled log-det increment $\Lambda_{\max}$ instead of an ambient~$\sqrt{d}$ term. We leave the proofs in Appendix~\ref{app:proof-main}.

However, the above regret bound relies on the trajectory-dependent $\widetilde\Lambda_T$. To further simplify the bound, we introduce the following definition.

\begin{definition}
\label{def:waterfill-info}
Let $G\succ0$ and $g_1,\ldots,g_d$ be its eigenvalues and let $B\ge0$. Define
\[
\Phi_G(B)
:=
\max_{\substack{a_j\ge0\\ \sum_{j=1}^d a_j\le B}}
\sum_{j=1}^d\log\left(1+\frac{a_j}{g_j}\right).
\]
Equivalently, let $\tau$ solve $\sum_{j=1}^d(\tau-g_j)_+=B$. Then
$\Phi_G(B)=\sum_{j:\,g_j<\tau}\log(\tau/g_j)$.
\end{definition}

\begin{theorem}[Spectral description for pooled regret]
\label{thm:spectral-pooled-envelope}
Let $L_{\max}:=\lceil\log_2 T\rceil+1$. Along any realized trajectory,
\[
\widetilde\Lambda_T
\le
L_{\max}\,\Phi_{G_{\mathrm{off}}}\!\left(\frac{T}{L_{\max}}\right).
\]
Denote
$c_{\mathrm{align}}
:=
\lambda_{\max}\!\left(G_{\mathrm{off}}^{1/2}M_{\mathrm{bias}}^{-1}G_{\mathrm{off}}^{1/2}\right)$ to be the largest eigenvalue,
then on the event of Theorem~\ref{thm:main}, the pooled branch satisfies, for
$\Gamma_{\mathrm{spec}}
:=\beta_{\mathrm{off}}+\sigma\sqrt{\Phi_{G_{\mathrm{off}}}(T)+2\log(1/\delta_{\min}^{\mathrm{pool}})}$,
\[
R_T
\le
8\bigl(\Gamma_{\mathrm{spec}}+\rho\sqrt{c_{\mathrm{align}}}\bigr)
\sqrt{
2T\,L_{\max}\,\Phi_{G_{\mathrm{off}}}\!\left(\frac{T}{L_{\max}}\right)
}.
\]
The full algorithm retains the online fallback $R_T\le C_{\mathrm{SL}}\sqrt{Td\,L_T^3}$.
\end{theorem}

Next, we give a specific regime in which the offline data are informative enough to reduce the online regret.

\begin{corollary}[Explicit pooled-better regime]
\label{cor:pooled-better-regime}
On the event of Theorem~\ref{thm:main}, the pooled branch satisfies
\[
R_T
\le
8\bigl(\Gamma_{\mathrm{pool}}+\rho\sqrt{c_{\mathrm{align}}}\bigr)
\sqrt{2T\,\widetilde\Lambda_T}.
\]
Consequently, if
\[
\widetilde\Lambda_T
\bigl(\Gamma_{\mathrm{pool}}+\rho\sqrt{c_{\mathrm{align}}}\bigr)^2
\le
\frac{C_{\mathrm{SL}}^2}{128}\,d\,L_T^3,
\]
then the pooled term in Theorem~\ref{thm:main} is no larger than the online fallback.
\end{corollary}

\begin{remark}
\label{rem:spectral-pooled-interpretation}
Theorem~\ref{thm:spectral-pooled-envelope} provides a deterministic regret bound. Offline gain enters only through the spectrum of $G_{\mathrm{off}}$ inside $\Phi_{G_{\mathrm{off}}}$, while the interaction of offline precision with the bias certificate is summarized by the single generalized-eigenvalue scale $c_{\mathrm{align}}$.
Equivalently, $c_{\mathrm{align}}$ is the smallest $c\ge0$ such that $G_{\mathrm{off}}M_{\mathrm{bias}}^{-1}G_{\mathrm{off}}\preceq c\,G_{\mathrm{off}}$.
Corollary~\ref{cor:pooled-better-regime} makes the improvement regime explicit: transfer is favorable whenever the residual pooled complexity $\widetilde\Lambda_T$ and the pooled prefactor $\Gamma_{\mathrm{pool}}+\rho\sqrt{c_{\mathrm{align}}}$ are jointly much smaller than the ambient online benchmark $dL_T^3$.
In the zero-bias case $\rho=0$, if only $r$ directions remain weak after offline preconditioning in the sense that $\widetilde\Lambda_T=O(r)$, $\Lambda_{\max}=O(1)$, and $\beta_{\mathrm{off}}=O(1)$ along a family of instances, then the pooled branch scales as $O(\sqrt{rT})$ up to logarithmic factors, while the online fallback remains $O(\sqrt{dT})$ up to logarithmic factors.
If the directional bias certificate is conservative but valid, the pooled term merely widens and the method falls back more often to the online branch.
If the radius is undersized so that the true bias is excluded, then the pooled branch should be read as a heuristic warm start rather than as part of the theorem's guarantee.
\end{remark}

\begin{theorem}
\label{thm:lower}
Consider the diagonal Gaussian subclass with actions $\{e_1,\dots,e_d\}$, arm-wise offline counts
$n_i$, and directional bias budgets $v_i:=\rho/\sqrt{b_i}$, equivalently
$G_{\mathrm{off}}=\operatorname{diag}(1+n_1,\dots,1+n_d)$ and
$M_{\mathrm{bias}}=\operatorname{diag}(b_1,\dots,b_d)$. For $S\subset[d]$, write
$N(S):=\sum_{i\in S}n_i$, $v(S):=\min_{i\in S}v_i$, and
\[
\Psi(S):=
\sqrt{\frac{|S|T^2}{T+N(S)}}+
\min\left\{\frac{N(S)T}{T+N(S)}\,v(S),\;
\sqrt{\frac{|S|N(S)T}{T+N(S)}}\right\}.
\]
Then there exist universal constants $c>0$ and $m_0\in \mathbb N$ such that, for every $d\ge m_0$,
\[
\inf_{\pi}\sup_{I\in \mathcal H_{\mathrm{diag}}}
\mathbb E_I[R_T(\pi)]
\ge
c\max_{\substack{S\subset[d]\\ |S|\ge m_0}}
\min\left\{\sqrt{dT},\,\Psi(S)\right\}.
\]
\end{theorem}

\begin{remark}
The lower bound carries the important information: transfer hardness is
governed by subset-level offline coverage~$N(S)$ and directional bias budgets~$v(S)$, rather than
by scalar summaries alone. Together, the upper and lower bounds show that offline-to-online
transfer in linear contextual bandits is inherently directional. The proof is deferred to
Appendix~\ref{app:proof-lower}.
\end{remark}

\section{Data-Driven Directional Bias Bounds}
\label{sec:epoch-wise-ell-minucb}

When a valid fixed directional bias certificate~$(M_
{\mathrm{bias}},\rho)$ is unavailable, how can the learner safely use the offline data? Our answer is to learn a conservative epoch-wise surrogate
certificate that can be plugged into the same algorithm while preserving the online learning safety.

The construction of such certificate combines an offline localization of $\theta^\dagger$ with an online localization of
$\theta_*$ from historical data. Their estimate gap provides a data-dependent center, while the
parallel-sum geometry records the precision shared by the two designs. Adding the corresponding ridge
confidence widths then yields a conservative certificate for $\theta_*-\theta^\dagger$.
Details and proofs are deferred to Appendices~\ref{app:epoch-learned-proofs}.

\paragraph{Epoch construction.}
Fix
\(
1=\tau_1<\cdots<\tau_{K_{\mathrm{ep}}}\le T
\)
with
\(
\tau_{K_{\mathrm{ep}}+1}:=T+1
\).
For each~$k$, before round~$\tau_k$, from past \emph{online} rounds only, define
\[
G_k^{\mathrm{on}}
:=
I+\sum_{s<\tau_k}x_{s,a_s}x_{s,a_s}^{\top},
\qquad
\hat\theta_k^{\mathrm{on}}
:=
(G_k^{\mathrm{on}})^{-1}\sum_{s<\tau_k}x_{s,a_s}r_s.
\]
The epoch directional bias certificate sets~$\widehat M_{k,\mathrm{bias}}$ to the \emph{parallel sum} of~$G_k^{\mathrm{on}}$ and~$G_{\mathrm{off}}$, which is
\(
\widehat M_{k,\mathrm{bias}}
:=\bigl((G_k^{\mathrm{on}})^{-1}+G_{\mathrm{off}}^{-1}\bigr)^{-1}
\),
and the radius
\[
\widehat\rho_k
:=
\|\hat\theta_k^{\mathrm{on}}-\hat\theta_{\mathrm{off}}\|_{\widehat M_{k,\mathrm{bias}}}
+
\beta_k^{\mathrm{on}}(\delta_k)
+
\beta_{\mathrm{off}}(\delta_{\mathrm{off}}),
\]
where $\beta_k^{\mathrm{on}}(\delta_k)$ is any valid ridge self-normalized bound for~$\theta_*$ in
the $G_k^{\mathrm{on}}$-norm.
\begin{remark}
\label{rem:parallel-sum-interpretation}
The parallel sum is canonical here. For any vector $x$, the quadratic form induced by
$\widehat M_{k,\mathrm{bias}}$ is the minimum joint precision cost of decomposing $x$ into an
online-certifiable component and an offline-certifiable component:
\[
\|x\|_{\widehat M_{k,\mathrm{bias}}}^2
=
\inf_{x=u+v}
\Bigl\{
\|u\|_{G_k^{\mathrm{on}}}^2+\|v\|_{G_{\mathrm{off}}}^2
\Bigr\}.
\]
Thus $\widehat M_{k,\mathrm{bias}}$ measures how cheaply a candidate bias vector can be certified
simultaneously by the online and offline confidence geometries. The above variational identity is proved in
Appendix~\ref{app:epoch-dd-details}.
\end{remark}
Rounds $t\in\{\tau_k,\ldots,\tau_{k+1}-1\}$ use the layerwise branch aggregation of
Section~\ref{sec:algorithm}, replacing the fixed-bound map $\psi_{t,\ell}$ by
$\widehat\psi_{k,t,\ell}$ in \eqref{eq:epoch-psi} and $\mathrm{rad}^{\mathrm{pool}}_{t,\ell}$ by
$\widehat{\mathrm{rad}}^{\mathrm{pool,ep}}_{t,\ell}$ in \eqref{eq:epoch-pool-rad} of
Appendix~\ref{app:epoch-dd-details}. The branch aggregation structure is otherwise unchanged, so the
online branch remains a safe fallback whenever the learned pooled term is loose. The next theorem records the resulting regret guarantee.
The full epoch-wise pseudocode is deferred to Appendix~\ref{app:epoch-dd-details}
(Algorithm~\ref{alg:epoch-ell-minucb}).

\begin{theorem}[Regret with epoch-wise learned directional bias certificates]
\label{thm:epoch-parallel-regret}
Run Algorithm~\ref{alg:epoch-ell-minucb}.
Let $k(t)$ be the epoch index satisfying
$t\in\{\tau_{k(t)},\ldots,\tau_{k(t)+1}-1\}$, and define the routed path sum
\[
\widetilde\Psi_T^{\mathrm{ep}}
:=
\sum_{t=1}^T\widehat\rho_{k(t)}\,\widehat\psi_{k(t),t,\ell_t}(x_{t,a_t}).
\]
Then, with probability at least $1-T^{-2}-\delta_{\mathrm{bias}}$,
\[
R_T
\le
\min\left\{
C_{\mathrm{SL}}\sqrt{Td\,L_T^3},
\;
8\sum_{\ell=0}^{\lceil\log_2 T\rceil}
\bigl(\beta_{\mathrm{off}}+\gamma_\ell^{\mathrm{fin}}\bigr)
\sqrt{2|Q_\ell|\,\Lambda_\ell}
+8\,\widetilde\Psi_T^{\mathrm{ep}}
\right\},
\]
where $Q_\ell$, $\Lambda_\ell$, and $\gamma_\ell^{\mathrm{fin}}$ are as in Appendix~\ref{app:proof-main}.
In particular,
\[
R_T
\le
\min\left\{
C_{\mathrm{SL}}\sqrt{Td\,L_T^3},
\;
8\Gamma_{\mathrm{pool}}\sqrt{2T\,\widetilde\Lambda_T}
+8\,\widetilde\Psi_T^{\mathrm{ep}}
\right\},
\]
where~$\Gamma_{\mathrm{pool}}$ is as in Theorem~\ref{thm:main}.
\end{theorem}

\noindent
Theorem~\ref{thm:epoch-parallel-regret} preserves the same min-structure as the fixed-certificate
result. When the learned certificate is informative, the pooled branch
improves on the online branch. When it is loose, the pooled term widens and the guarantee falls back
toward the purely online rate.
To obtain a closed-form refinement, we specialize to a standard doubling schedule. Since this
schedule is only a default instantiation, we defer its precise definition and the associated
confidence split to Appendix~\ref{app:epoch-dd-details} and state here only the resulting corollary.

\begin{corollary}[Adaptivity penalty under doubling epochs]
\label{cor:epoch-adaptivity-penalty}
Under the doubling schedule and confidence split of
Theorem~\ref{thm:epoch-doubling-regret}, define
\[
\widehat c_k
:=
\lambda_{\max}\bigl(
G_{\mathrm{off}}^{1/2}\widehat M_{k,\mathrm{bias}}^{-1}G_{\mathrm{off}}^{1/2}
\bigr),
\qquad
\mathcal S_{\mathrm{ep}}
:=
\sum_{k=1}^{K_{\mathrm{ep}}} (\tau_{k+1}-\tau_k)\,\widehat\rho_k^2\,\widehat c_k.
\]
Then, with probability at least $1-T^{-2}-\delta_{\mathrm{bias}}$,
\[
R_T
\le
\min\left\{
C_{\mathrm{SL}}\sqrt{Td\,L_T^3},
\;
8\Bigl(\Gamma_{\mathrm{pool}}+\sqrt{\mathcal S_{\mathrm{ep}}/T}\Bigr)
\sqrt{2T\,\widetilde\Lambda_T}
\right\}.
\]
\end{corollary}

\begin{remark}
Corollary~\ref{cor:epoch-adaptivity-penalty} isolates the price of learning the certificate from
data. The fixed-certificate pooled branch pays the deterministic prefactor
$\Gamma_{\mathrm{pool}}+\rho\sqrt{c_{\mathrm{align}}}$, whereas the learned version pays
$\Gamma_{\mathrm{pool}}+\sqrt{\mathcal S_{\mathrm{ep}}/T}$.
Thus the learned rule approaches fixed-certificate behavior when the epoch-wise estimator gap
stabilizes quickly and the learned geometries $\widehat M_{k,\mathrm{bias}}$ stay well aligned with
$G_{\mathrm{off}}$, so that the average penalty $\mathcal S_{\mathrm{ep}}/T$ is small.
Conversely, persistent disagreement between online and offline estimates in highly covered directions
enlarges $\widehat\rho_k$ or $\widehat c_k$, widens the learned pooled term, and drives the method
back toward the online fallback.
\end{remark}

\section{Experiments}
\label{sec:experiments}
\FloatBarrier

We visualize how regret responds to increasing coordinate-split mismatch between $\theta_*$ and $\theta^\dagger$ under randomized contexts. The main experiment uses
$K{=}5$ arms and $d{=}5$ features with normalized random contexts. We fix
$\theta_*=(s,1,1,1,1)$ and $\theta^\dagger=(1,s,1,1,1)$ with
$s\in\{1.1,2,10\}$ in the three panels of Figure~\ref{fig:exp-main}. These regimes move from mild to
severe coordinate-split mismatch: the online parameter puts extra mass on the first coordinate while
the offline parameter puts matching mass on the second.

We compare a standard SupLinUCB baseline, Ellipsoidal-MINUCB, 3D Ellipsoidal-MINUCB(Algorithm~\ref{alg:epoch-ell-minucb}), a warm-started
SupLinUCB, and a non-adaptive policy that uses only the offline estimate. Full hyperparameters and
implementation details are in Appendix~\ref{app:more-experiments}.

\begin{figure}[H]
  \centering
  \includegraphics[width=\linewidth]{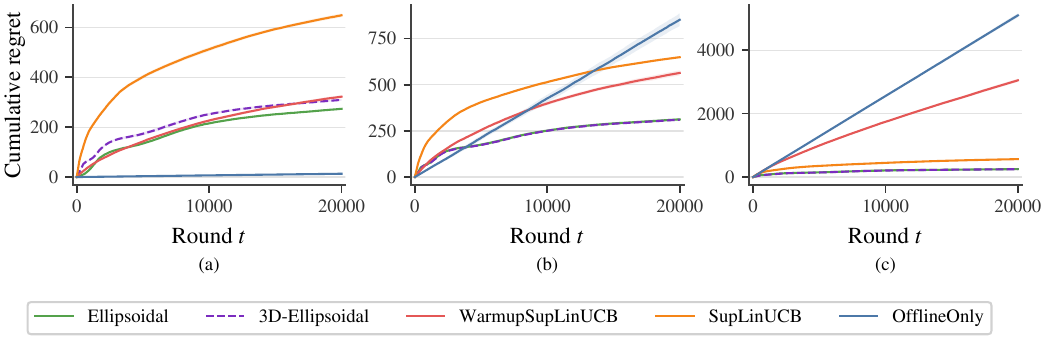}
  \caption{Cumulative regret curves}
  \label{fig:exp-main}
\end{figure}

Figure~\ref{fig:exp-main} shows that larger $s$ increases mismatch between the highlighted coordinates.
The ordering of curves is consistent with the min-structure in Theorem~\ref{thm:main}, and the
adaptive variant tracks the same phases with additional conservatism.

\paragraph{Additional experiments.}
Appendix~\ref{app:more-experiments} reports three additional experiments: randomized bias generation that
trades a multiplicative component against an additive offset; fixed $\theta_*$ with varying alignment of
the offline mean; and \emph{offline} covariate support held strictly inside the \emph{online} support, so
that the offline design matrix is rank-deficient in some directions when the two parameters align.
\FloatBarrier

\section{Related Work}
\label{sec:related}

A broad offline-to-online literature studies how historical data, hints, or source models can
reduce the cost of online learning. Warm-start and transfer approaches combine supervised and bandit
feedback \citep{zhang2019}, use initial hints \citep{cutkosky2022}, handle deficient support in
logged data \citep{tran2021}, and transfer across contextual bandit tasks \citep{cai2024}. Other
work considers weighted hypothesis transfer \citep{bilaj2022}, partially observed confounded
offline data \citep{tennenholtz2020}, offline logs in larger sequential settings
\citep{wagenmaker2023,kausik2024}, and offline-to-online interaction in pricing
\citep{zhang2025contextualpricing,bu2020icml,bu2024featurepricing}. A related offline contextual
bandit and off-policy learning line learns or evaluates policies from logged feedback under support,
variance, or pessimism constraints \citep{dudik2011,swaminathan2015,wang2024,gabbianelli2024}. These
works establish many ways in which past data can help decision-making, but their main object is not
the online regret value of biased offline evidence after deployment.

The closest comparisons come from two directions. In linear bandits with offline data, prior
observations can lower exploration cost through offline Gram matrices and design-based exploration
when the offline and online parameters coincide \citep{vijayan2025oope}. In multi-armed bandits with
biased offline data, MIN-UCB shows that a learner should exploit offline samples only when the
mismatch is controlled, and otherwise fall back toward online learning \citep{cheung2024}; related
biased offline-to-online questions have also been studied in combinatorial and dueling bandit
variants \citep{zhou2025hybrid,he2025learninggap}. \citet{zhou2026regret} studied the biased offline to online setting under noncontextual linear bandit model where the bias is modeled through a known Euclidean bound which fails to capture the directional information in the offline data. These lines provide the two ingredients we use:
coverage in a linear model and controlled use of biased offline information. Our focus is matrix-shaped bias budgets together with offline design geometry in a shared-feature contextual model, rather than arm-level transfer alone.

At the algorithmic level, our online branch builds on the optimism-based foundations of
contextual and linear bandits: LinUCB and SupLinUCB \citep{li2010,chu2011}, the OFUL analysis
\citep{abbasi2011}, and linear Thompson-style approaches \citep{agrawal2013}. We also relate to
misspecification-robust bandits, where the online reward model may deviate from the assumed linear
class. Classic and follow-up work includes \citet{ghosh2017} and \citet{takemura2021}; more recent
developments include \citet{foster2020misspec} and \citet{krishnamurthy2021}. Here the online model
remains linear; the difficulty is that the \emph{offline} parameter may differ from the online one in
a structured, direction-dependent way. Our contribution is to encode that mismatch through a
directional bias certificate and use it inside a hybrid pooled/online optimistic rule. For the adaptive
part, we are not aware of prior regret guarantees that learn a directional bias certificate from biased
offline data in linear contextual bandits.

\section{Conclusion}
\label{sec:conclusion}

This paper investigates how to use biased offline data in linear contextual bandits. We propose a novel direction-aware framework to measure the bias of the offline data and use them safely with comprehensive theoretical analysis and empirical validation. It makes bandit algorithms easier to apply in practice by relaxing how historical data must relate to the live environment. Future directions include extensions to other online learning settings, such as generalized linear models and online Markov decision processes. Also, it would also be interesting to develop refined lower bounds in the non-diagonal settings which are more challenging.

\bibliographystyle{plainnat}
\bibliography{references}

\clearpage
\appendix
\section{Intuition behind the directional bias certificate}
\label{app:directional-bias-bound}

In this section, we discuss the intuition behind the directional bias certificate in detail. The core idea is to express the error $\Delta:=\theta_*-\theta^\dagger$ in the linear combination of the eigenvectors of $M_{\mathrm{bias}}$ and get a clear view of the bias along each specific eigenvector.

Let
\[
\Delta:=\theta_*-\theta^\dagger.
\]
The directional bias certificate
\[
\|\Delta\|_{M_{\mathrm{bias}}}\le \rho
\]
is equivalent to
\[
\Delta^\top M_{\mathrm{bias}}\Delta\le \rho^2.
\]
We can write the eigen-decomposition of $M_{\mathrm{bias}}$ as
\[
M_{\mathrm{bias}}=Q\Lambda Q^\top,
\qquad
Q=[q_1,\ldots,q_d],
\qquad
Q^\top Q=I,
\qquad
\Lambda=\mathrm{diag}(\lambda_1,\ldots,\lambda_d),
\]
where $\lambda_j>0$ and $q_j$ are the eigenvectors of $M_{\mathrm{bias}}$. Define
\[
\alpha:=Q^\top\Delta,
\qquad
\alpha_j=q_j^\top\Delta.
\]
Here, $\alpha_j$ can be viewed as the error projected to the $j$-th eigenvector. Then we have the following expression:
$$
\qquad
\Delta=Q\alpha=\sum_{j=1}^d \alpha_j q_j.$$
By definition, we have
\[
\Delta^\top M_{\mathrm{bias}}\Delta
=
\Delta^\top Q\Lambda Q^\top\Delta
=
(Q^\top\Delta)^\top\Lambda(Q^\top\Delta)
=
\alpha^\top\Lambda\alpha
=
\sum_{j=1}^d \lambda_j\alpha_j^2.
\]
Therefore
\[
\|\Delta\|_{M_{\mathrm{bias}}}\le \rho
\quad\Longleftrightarrow\quad
\sum_{j=1}^d \lambda_j\alpha_j^2\le \rho^2.
\]
In particular, for each eigen-direction $q_j$,
\[
|\alpha_j|=|q_j^\top\Delta|\le \frac{\rho}{\sqrt{\lambda_j}}.
\]
Till now, we can see that a larger $\lambda_j$ typically means a small error along $q_j$. Under such circumstance, the offline data should be informative along $q_j$. Otherwise, the offline data is misleading along $q_j$.

For any feature direction $x$, we have
$$|x^{\top}(\theta_*-\theta^\dagger)| \le \rho \|x\|_{M_{\mathrm{bias}}^{-1}}.$$

The right hand side in the above inequality can be written as:
$$\|x\|_{M_{\mathrm{bias}}^{-1}}^2
=
\sum_{j=1}^d \frac{(q_j^\top x)^2}{\lambda_j}.$$
It means that directions aligned with small eigenvalues are amplified by $M_{\mathrm{bias}}^{-1}$, while directions
aligned with large eigenvalues are shrunk. Thus a direction with most of its energy on small-$\lambda_j$ eigendirections makes the above bound looser leading to a large worst-case bias width in the analysis of the UCB radius.

\section{An Impossibility Result}
\label{app:isotropic-too-coarse}

The main text models bias through a directional bias certificate
$\|\theta_*-\theta^\dagger\|_{M_{\mathrm{bias}}}\le \rho$. The next theorem explains why replacing
this by a single Euclidean summary $\|\theta_*-\theta^\dagger\|_2\le \rho$ is too coarse.  Under a finite-action Bernoulli linear bandit with dimension
$d=3$, parameter norm bound $S=1$, and $|\cA|=3$, we have the following theorem:

\begin{theorem}[Isotropic bounds do not identify transferable directions]
\label{thm:rho-only-nonident}
There exist universal constants $c_0,c_1>0$ such that the following holds. Fix
$T\ge 1$, $\rho_0\in(0,1/4]$, and
\[
0 < \rho \le \min\!\left\{\frac{\rho_0}{4},\, c_0 T^{-1/2}\right\}.
\]
Then there exist three Bernoulli linear bandit instances
$I_+,I_-,I_{\mathrm{good}}$ sharing a common offline parameter $\theta^\dagger$ and a common
action set $\cA_t\equiv\cA$, such that, writing $\theta_*^I$ for the online parameter of instance~$I$,
the offline data $D_{\mathrm{off}}$ has the same distribution under
$I_+,I_-,I_{\mathrm{good}}$, and
\[
\norm{\theta_*^I-\theta^\dagger}_2\le \rho
\quad
\forall I\in\{I_+,I_-,I_{\mathrm{good}}\}.
\]
Nevertheless, the corresponding online problems have different hardness:
\[
\inf_{\pi}
\sup_{I\in\{I_+,I_-\}}
\E_I[R_T(\pi)]
\ge
c_1 \frac{\rho^2 T}{\rho_0};
\]
where the infimum is over all non-anticipatory policies. In contrast, there exists a deterministic
fixed-action policy $\pi_{\mathrm{good}}$ such that
\[
\E_{I_{\mathrm{good}}}[R_T(\pi_{\mathrm{good}})]=0.
\]
Consequently, a $\rho$-only isotropic summary cannot determine which offline directions transfer.
\end{theorem}

\begin{proof}
Let $e_0,e_1,e_2$ denote the canonical basis of $\R^3$, and set
\[
\bar{\rho} := \frac{\rho}{\sqrt{2}}.
\]
Define the three actions
\[
a_0 := e_1,
\qquad
a_+ := \frac{\rho_0 e_1 + \bar{\rho} e_2}{\sqrt{\rho_0^2+\bar{\rho}^2}},
\qquad
a_- := \frac{\rho_0 e_1 - \bar{\rho} e_2}{\sqrt{\rho_0^2+\bar{\rho}^2}},
\]
and lift them to
\[
x^{(0)} := \frac{e_0+a_0}{\sqrt{2}},
\qquad
x^{(+)} := \frac{e_0+a_+}{\sqrt{2}},
\qquad
x^{(-)} := \frac{e_0+a_-}{\sqrt{2}}.
\]
Each lifted action has norm $1$, so the finite action set
\[
\cA_t := \{x^{(0)},x^{(+)},x^{(-)}\}
\qquad \forall t
\]
fits the main setup with $K=3$.

Now define the common offline parameter and the three online parameters by
\[
\tilde\theta^\dagger
:=
\frac{1}{\sqrt{2}}e_0 + \sqrt{2}\rho_0 e_1,
\]
\[
\tilde\theta_+
:=
\frac{1}{\sqrt{2}}e_0 + \sqrt{2}\bigl(\rho_0 e_1 + \bar{\rho} e_2\bigr),
\qquad
\tilde\theta_-
:=
\frac{1}{\sqrt{2}}e_0 + \sqrt{2}\bigl(\rho_0 e_1 - \bar{\rho} e_2\bigr),
\]
\[
\tilde\theta_{\mathrm{good}}
:=
\frac{1}{\sqrt{2}}e_0 + \sqrt{2}(\rho_0+\bar{\rho})e_1.
\]
Write these parameters as
$\theta_+,\theta_-,\theta_{\mathrm{good}}$, with common offline parameter $\theta^\dagger$.
Take any deterministic offline covariates $\{z_i\}$ with $\|z_i\|_2\le 1$ and
$z_i^\top \theta^\dagger \in [0,1]$, and let the offline rewards be Bernoulli with means
$z_i^\top \theta^\dagger$. Online rewards are also Bernoulli, with mean $x^\top \theta$ when action
$x\in\cA_t$ is played under parameter $\theta$.

All displayed online means lie in $[1/2,13/16]$: the smallest is
$1/2 + (\rho_0^2-\bar{\rho}^2)/\sqrt{\rho_0^2+\bar{\rho}^2} \ge 1/2$, and the largest is
$1/2+\rho_0+\bar{\rho} \le 13/16$ because $\rho_0\le 1/4$ and $0<\rho\le \rho_0/4$.
Hence rewards are in $[0,1]$, and the centered Bernoulli noise is conditionally sub-Gaussian.

The three instances share the same offline parameter $\theta^\dagger$ and the same deterministic
offline covariates, so the law of $D_{\mathrm{off}}$ is identical across all three instances. Also,
\[
\norm{\theta_+ - \theta^\dagger}_2
=
\norm{\theta_- - \theta^\dagger}_2
=
\norm{\theta_{\mathrm{good}} - \theta^\dagger}_2
=
\rho,
\]
so the same isotropic directional bias certificate is valid throughout.

For the hard family, let $P_+$ and $P_-$ denote the laws of the full online interaction under
$\theta_+$ and $\theta_-$, respectively. Let $v:=e_2$, and define the event
\[
E:=\left\{\sum_{t=1}^T \mathbf{1}\{\ip{v}{x_t}\ge 0\}\ge \frac{T}{2}\right\}.
\]
Under $\theta_-$, the unique optimal action in $\cA_t$ is $x^{(-)}$, whose mean
reward equals
\[
(x^{(-)})^\top \theta_- = \frac{1}{2} + \sqrt{\rho_0^2+\bar{\rho}^2}.
\]
For any chosen action $x_t\in\cA_t$, write
\[
x_t = \frac{e_0+a_t}{\sqrt{2}}
\qquad\text{with}\qquad
a_t\in\{a_0,a_+,a_-\}.
\]
Any action with $\ip{v}{x_t}\ge 0$ is either $x^{(0)}$ or $x^{(+)}$, and each satisfies
\[
x_t^\top \theta_-
=
\frac{1}{2} + \rho_0 \ip{e_1}{a_t} - \bar{\rho}\ip{v}{a_t}
\le
\frac{1}{2} + \rho_0.
\]
Hence every round in $E$ incurs regret at least
\[
\sqrt{\rho_0^2+\bar{\rho}^2} - \rho_0
\ge
\frac{\bar{\rho}^2}{4\rho_0}
=
\frac{\rho^2}{8\rho_0},
\]
where the last inequality uses $\bar{\rho} \le \rho_0/4$. Therefore,
\[
R_T(\theta_-)\ge \frac{\rho^2 T}{16\rho_0}
\qquad \text{on } E.
\]
By symmetry,
\[
R_T(\theta_+)\ge \frac{\rho^2 T}{16\rho_0}
\qquad \text{on } E^c.
\]

The offline KL divergence between the two hard instances is zero because the offline law is the same.
For the online interaction, if the chosen action at round $t$ has Bernoulli means
$p_t:=x_t^\top\theta_+$ and $q_t:=x_t^\top\theta_-$, then
\[
\mathrm{kl}(p_t,q_t)
\le
\frac{(p_t-q_t)^2}{q_t(1-q_t)}
\le
\frac{256}{39}(p_t-q_t)^2,
\]
because $q_t\in[1/2,13/16]$. Moreover, $p_t-q_t=0$ for $x_t=x^{(0)}$, while for
$x_t\in\{x^{(+)},x^{(-)}\}$,
\[
|p_t-q_t|
=
\left|x_t^\top(\theta_+-\theta_-)\right|
=
\frac{2\bar{\rho}^2}{\sqrt{\rho_0^2+\bar{\rho}^2}}
\le
\frac{2\bar{\rho}^2}{\rho_0}.
\]
Hence every round contributes at most
\[
\mathrm{KL}(p_t,q_t)
\le
\frac{1024}{39}\frac{\bar{\rho}^4}{\rho_0^2}
\le
\frac{64}{39}\bar{\rho}^2
=
\frac{32}{39}\rho^2,
\]
where the second inequality again uses $\bar{\rho}\le \rho_0/4$. By the chain rule,
\[
\mathrm{KL}(P_+\,\|\,P_-)\le \frac{32}{39}\rho^2T.
\]
Thus $\mathrm{KL}(P_+\,\|\,P_-)\le 1/8$ when $c_0$ is chosen small enough. Bretagnolle--Huber then
gives
\[
P_+(E^c)+P_-(E)\ge \frac{1}{2}e^{-1/8}.
\]
Combining the last three displays yields
\[
\max\{\E_+[R_T],\, \E_-[R_T]\}
\ge
\frac{e^{-1/8}}{64}\frac{\rho^2 T}{\rho_0}.
\]
This yields the first claim after absorbing constants into $c_1$.

For the good instance, the action $x^{(0)}$ has mean
\[
(x^{(0)})^\top \theta_{\mathrm{good}}
=
\frac{1}{2}+\rho_0+\bar{\rho},
\]
whereas
\[
(x^{(\pm)})^\top \theta_{\mathrm{good}}
=
\frac{1}{2}
+
\frac{\rho_0(\rho_0+\bar{\rho})}{\sqrt{\rho_0^2+\bar{\rho}^2}}
<
\frac{1}{2}+\rho_0+\bar{\rho}.
\]
Thus $x^{(0)}$ is optimal at every round, and the deterministic policy that always plays $x^{(0)}$
has zero regret, yielding the second claim.

$D_{\mathrm{off}}$ has the same law across the good and hard instances, so any measurable summary
$S=\Phi(D_{\mathrm{off}})$ also has the same distribution. Hence no offline-only summary can
distinguish the easy instance from the hard family.
\end{proof}

\section{Notation}
\label{app:notation}

Table~\ref{tab:notation} collects recurrent symbols. The layerwise procedure is stated as
Algorithm~\ref{alg:ellipsoidal-minucb} in the main text.

\begin{table}[ht]
\centering
\caption{Main notation.}
\label{tab:notation}
\small
\setlength{\tabcolsep}{5pt}
\begin{tabular}{@{}p{0.26\linewidth}p{0.68\linewidth}@{}}
\toprule
\textbf{Symbol} & \textbf{Meaning} \\
\midrule
$T,K,d$ & Horizon, number of actions per round, feature dimension. \\
$\theta_*$ & Unknown online reward parameter; $\|\theta_*\|_2\le S$. \\
$\theta^\dagger$ & Offline regression parameter in $y_i=z_i^\top\theta^\dagger+\xi_i$. \\
$D_{\mathrm{off}},n_{\mathrm{off}}$ & Offline sample $\{(z_i,y_i)\}$ and its size. \\
$G_{\mathrm{off}},\hat\theta_{\mathrm{off}},\beta_{\mathrm{off}}(\delta)$ & Offline ridge design
$I+\sum_i z_i z_i^\top$, ridge estimator, and offline concentration radius. \\
$(M_{\mathrm{bias}},\rho)$ & Directional bias certificate: $\|\theta_*-\theta^\dagger\|_{M_{\mathrm{bias}}}\le\rho$. \\
$\cF_{t-1}$ & Filtration generated by past data and current contexts (Section~\ref{sec:setup}). \\
$\ell,L$ & SupLinUCB layer index; $L=\lceil\log_2 T\rceil$. \\
$\cH_{t,\ell}$ & History routed to layer $\ell$ before round $t$. \\
$A_{t,\ell},V_{t,\ell},b_{t,\ell}$ & Layerwise online Gram matrix, $V_{t,\ell}=I+A_{t,\ell}$, and response vector
$b_{t,\ell}$ on layer $\ell$ (Section~\ref{sec:algorithm}). \\
$\hat\theta^{\mathrm{on}}_{t,\ell},\hat\theta^{\mathrm{pool}}_{t,\ell}$ & Online and pooled ridge estimators on layer $\ell$. \\
$\beta^{\mathrm{on}}_{t,\ell},\gamma_{t,\ell}$ & Online SupLin radius and pooled martingale radius. \\
$\psi_{t,\ell}(x)$ & Bias routing term in the pooled radius (Section~\ref{sec:algorithm}). \\
$U^{\mathrm{on}},U^{\mathrm{pool}},U^{\min},L^{\max},w$ & Branch UCB/LCB scores and aggregated width. \\
$\delta_{\mathrm{off}},\delta_{\mathrm{on}},\delta_{t,\ell}^{\mathrm{on}},\delta_{t,\ell}^{\mathrm{pool}}$ & Offline and online union budgets; per-$(t,\ell)$ online and pooled budgets (Theorem~\ref{thm:main}). \\
$\widetilde\Lambda_T,\widetilde\Psi_T$ & Offline design log-det sum and path sum of $\psi$ (Theorem~\ref{thm:main}). \\
$A_\ell^{\mathrm{fin}}$ & Final Gram matrix on layer $\ell$ at horizon $T$. \\
$\Lambda_{\max},\Gamma_{\mathrm{pool}},\Gamma_{\mathrm{spec}}$ & Pooled log-det and spectral shorthands in Theorems~\ref{thm:main} and~\ref{thm:spectral-pooled-envelope}. \\
\bottomrule
\end{tabular}
\end{table}

\section{Proof of Proposition~\ref{prop:offline-region}}

\begin{proof}
Condition on the offline covariates $\{z_i\}_{i=1}^{n_{\mathrm{off}}}$. The imported
self-normalized ridge inequality states that, with probability at least $1-\delta$,
\[
\|\hat\theta_{\mathrm{off}}-\theta^\dagger\|_{G_{\mathrm{off}}}
\le
\beta_{\mathrm{off}}(\delta).
\]
On the same event, the directional bias certificate gives
\[
\|\theta_*-\theta^\dagger\|_{M_{\mathrm{bias}}}\le \rho.
\]
Write
\[
\theta_*-\hat\theta_{\mathrm{off}}
=
(\theta^\dagger-\hat\theta_{\mathrm{off}})+(\theta_*-\theta^\dagger)
=:u+v.
\]
Then
\[
\|u\|_{G_{\mathrm{off}}}
=
\|\hat\theta_{\mathrm{off}}-\theta^\dagger\|_{G_{\mathrm{off}}}
\le \beta_{\mathrm{off}}(\delta),
\qquad
\|v\|_{M_{\mathrm{bias}}}
=
\|\theta_*-\theta^\dagger\|_{M_{\mathrm{bias}}}
\le \rho.
\]
Hence
\[
\theta_*=\hat\theta_{\mathrm{off}}+u+v
\]
with $u$ and $v$ belonging to the two ellipsoids claimed in the proposition, which proves the region
inclusion.
Let
\[
\theta=\hat\theta_{\mathrm{off}}+u+v
\]
with $\|u\|_{G_{\mathrm{off}}}\le \beta_{\mathrm{off}}$ and $\|v\|_{M_{\mathrm{bias}}}\le \rho$.
Then
\[
x^\top \theta
=
x^\top \hat\theta_{\mathrm{off}} + x^\top u + x^\top v.
\]
By Cauchy--Schwarz in the $G_{\mathrm{off}}$ metric,
\[
x^\top u
=
(G_{\mathrm{off}}^{-1/2}x)^\top(G_{\mathrm{off}}^{1/2}u)
\le
\|x\|_{G_{\mathrm{off}}^{-1}}\|u\|_{G_{\mathrm{off}}}
\le
\beta_{\mathrm{off}}\|x\|_{G_{\mathrm{off}}^{-1}}.
\]
Likewise,
\[
x^\top v
=
(M_{\mathrm{bias}}^{-1/2}x)^\top(M_{\mathrm{bias}}^{1/2}v)
\le
\|x\|_{M_{\mathrm{bias}}^{-1}}\|v\|_{M_{\mathrm{bias}}}
\le
\rho\|x\|_{M_{\mathrm{bias}}^{-1}}.
\]
This yields the upper bound
\[
x^\top \theta
\le
x^\top\hat\theta_{\mathrm{off}}
+\beta_{\mathrm{off}}\|x\|_{G_{\mathrm{off}}^{-1}}
+\rho\|x\|_{M_{\mathrm{bias}}^{-1}}.
\]
Equality holds by choosing $u$ and $v$ aligned with the dual directions
$G_{\mathrm{off}}^{-1}x$ and $M_{\mathrm{bias}}^{-1}x$, respectively; hence the display gives the
exact support function.
\end{proof}

\begin{lemma}[Support function of the offline region]
\label{lem:support}
Under the event of Proposition~\ref{prop:offline-region}, every $x\in\R^d$ satisfies
\[
\sup_{\theta \in \cC_{\mathrm{off}}} x^{\top}\theta
=
x^{\top}\hat\theta_{\mathrm{off}}
+ \beta_{\mathrm{off}}\|x\|_{G_{\mathrm{off}}^{-1}}
+ \rho\|x\|_{M_{\mathrm{bias}}^{-1}}.
\]
\end{lemma}

\begin{proof}
This is exactly the support-function computation established in the proof of
Proposition~\ref{prop:offline-region}.
\end{proof}

\section{Technical Lemmas}
\label{app:self-normalized}

\begin{lemma}[Self-normalized martingale term]
\label{lem:self-normalized-noise}
Let $(\cF_s)_{s=0}^n$ be a filtration. Suppose $\phi_s \in \R^d$ is $\cF_{s-1}$-measurable with
$\|\phi_s\|_2\le 1$, and $\eta_s$ is conditionally $\sigma$-sub-Gaussian given $\cF_{s-1}$. Define
\[
Z_t:=\sum_{s=1}^t \phi_s \eta_s,
\qquad
V_t:=I+\sum_{s=1}^t \phi_s\phi_s^\top.
\]
Then, for every $\delta\in(0,1)$,
\[
\Prb\!\left(
\|Z_n\|_{V_n^{-1}}
\le
\sigma\sqrt{2\log\frac{\det(V_n)^{1/2}}{\delta}}
\right)
\ge 1-\delta.
\]
\end{lemma}

\begin{proof}
Fix $\nu\in\R^d$ and define
\[
M_t(\nu)
:=
\exp\left(
\sigma^{-1}\ip{\nu}{Z_t}
-\frac12 \nu^\top(V_t-I)\nu
\right).
\]
Since $\phi_t$ is $\cF_{t-1}$-measurable and $\eta_t$ is conditionally $\sigma$-sub-Gaussian,
\[
\E\!\left[
\exp\left(\sigma^{-1}\eta_t\ip{\nu}{\phi_t}\right)\middle|\cF_{t-1}
\right]
\le
\exp\left(\frac12\ip{\nu}{\phi_t}^2\right).
\]
Therefore
\[
\E\!\left[\frac{M_t(\nu)}{M_{t-1}(\nu)}\middle|\cF_{t-1}\right]\le 1,
\]
so $(M_t(\nu))_{t=0}^n$ is a nonnegative supermartingale and $\E[M_n(\nu)]\le 1$.

Now mix over $\nu$ with the Gaussian density
\[
f(\nu)
:=
\frac{1}{(2\pi)^{d/2}}
\exp\left(-\frac12\nu^\top\nu\right),
\]
and define
\[
\widetilde M_n:=\int_{\R^d} M_n(\nu)f(\nu)\,d\nu.
\]
By Fubini,
\[
\E[\widetilde M_n]
=
\int_{\R^d} \E[M_n(\nu)]f(\nu)\,d\nu
\le
\int_{\R^d} f(\nu)\,d\nu
=1.
\]
Combining the definitions of $M_n(\nu)$ and $f(\nu)$ gives
\[
\widetilde M_n
=
\frac{1}{(2\pi)^{d/2}}
\int_{\R^d}
\exp\left(
\sigma^{-1}\ip{\nu}{Z_n}-\frac12\nu^\top V_n\nu
\right)d\nu.
\]
Completing the square,
\[
\sigma^{-1}\ip{\nu}{Z_n}-\frac12\nu^\top V_n\nu
=
-\frac12\left\|\nu-\sigma^{-1}V_n^{-1}Z_n\right\|_{V_n}^2
+ \frac{1}{2\sigma^2}\|Z_n\|_{V_n^{-1}}^2.
\]
Hence
\[
\widetilde M_n
=
\frac{1}{(2\pi)^{d/2}}
\exp\left(\frac{1}{2\sigma^2}\|Z_n\|_{V_n^{-1}}^2\right)
\int_{\R^d}
\exp\left(
-\frac12\left\|\nu-\sigma^{-1}V_n^{-1}Z_n\right\|_{V_n}^2
\right)d\nu.
\]
The remaining Gaussian integral equals $(2\pi)^{d/2}\det(V_n)^{-1/2}$, so
\[
\widetilde M_n
=
\frac{1}{\det(V_n)^{1/2}}
\exp\left(\frac{1}{2\sigma^2}\|Z_n\|_{V_n^{-1}}^2\right).
\]
Markov's inequality and $\E[\widetilde M_n]\le 1$ imply
\[
\Prb\!\left(\widetilde M_n\ge \delta^{-1}\right)\le \delta.
\]
Rearranging the event $\widetilde M_n\le \delta^{-1}$ yields the claim.
\end{proof}

\begin{lemma}[Self-normalized martingale with an SPD base matrix]
\label{lem:self-normalized-noise-base}
Let $(\cF_s)_{s=0}^n$ be a filtration. Suppose $\phi_s \in \R^d$ is $\cF_{s-1}$-measurable with
$\|\phi_s\|_2\le 1$, and $\eta_s$ is conditionally $\sigma$-sub-Gaussian given $\cF_{s-1}$.
Let $H\in\mathbb S_{++}^d$ be deterministic or, more generally, $\cF_0$-measurable. Define
\[
Z_t:=\sum_{s=1}^t \phi_s \eta_s,
\qquad
V_t:=H+\sum_{s=1}^t \phi_s\phi_s^\top.
\]
Then, for every $\delta\in(0,1)$,
\[
\Prb\!\left(
\|Z_n\|_{V_n^{-1}}
\le
\sigma\sqrt{2\log\frac{\det(V_n)^{1/2}}{\det(H)^{1/2}\delta}}
\right)
\ge 1-\delta.
\]
\end{lemma}

\begin{proof}
Condition on $H$ and fix $\nu\in\R^d$. Define
\[
M_t(\nu)
:=
\exp\left(
\sigma^{-1}\ip{\nu}{Z_t}
-\frac12 \nu^\top(V_t-H)\nu
\right).
\]
Exactly as in the proof of Lemma~\ref{lem:self-normalized-noise},
\[
\E\!\left[\frac{M_t(\nu)}{M_{t-1}(\nu)}\middle|\cF_{t-1}\right]\le 1,
\]
so $(M_t(\nu))_{t=0}^n$ is a nonnegative supermartingale and
$\E[M_n(\nu)\mid H]\le 1$.

Now mix over $\nu$ with the Gaussian density having covariance $H^{-1}$,
\[
f_H(\nu)
:=
\frac{\det(H)^{1/2}}{(2\pi)^{d/2}}
\exp\left(-\frac12\nu^\top H\nu\right),
\]
and define
\[
\widetilde M_n:=\int_{\R^d} M_n(\nu)f_H(\nu)\,d\nu.
\]
By Fubini,
\[
\E[\widetilde M_n\mid H]
=
\int_{\R^d}\E[M_n(\nu)\mid H]f_H(\nu)\,d\nu
\le 1.
\]
Combining the definitions gives
\[
\widetilde M_n
=
\frac{\det(H)^{1/2}}{(2\pi)^{d/2}}
\int_{\R^d}
\exp\left(
\sigma^{-1}\ip{\nu}{Z_n}-\frac12\nu^\top V_n\nu
\right)d\nu.
\]
Completing the square,
\[
\sigma^{-1}\ip{\nu}{Z_n}-\frac12\nu^\top V_n\nu
=
-\frac12\left\|\nu-\sigma^{-1}V_n^{-1}Z_n\right\|_{V_n}^2
+ \frac{1}{2\sigma^2}\|Z_n\|_{V_n^{-1}}^2.
\]
The remaining Gaussian integral equals $(2\pi)^{d/2}\det(V_n)^{-1/2}$, so
\[
\widetilde M_n
=
\frac{\det(H)^{1/2}}{\det(V_n)^{1/2}}
\exp\left(\frac{1}{2\sigma^2}\|Z_n\|_{V_n^{-1}}^2\right).
\]
Markov's inequality conditional on $H$ implies
\[
\Prb\!\left(\widetilde M_n\ge \delta^{-1}\middle|H\right)\le \delta.
\]
Rearranging the event $\widetilde M_n\le \delta^{-1}$ yields the claim.
\end{proof}

\begin{proposition}[Ridge confidence ellipsoid]
\label{prop:derive-self-normalized}
Under the assumptions of Lemma~\ref{lem:self-normalized-noise}, let
\[
Y_s=\phi_s^\top \theta + \eta_s,
\qquad
\hat\theta_n:=V_n^{-1}\sum_{s=1}^n \phi_sY_s,
\qquad
\|\theta\|_2\le S.
\]
Then, with probability at least $1-\delta$,
\[
\|\hat\theta_n-\theta\|_{V_n}
\le
\sigma\sqrt{2\log\frac{\det(V_n)^{1/2}}{\delta}}
+S.
\]
\end{proposition}

\begin{proof}
Using $Y_s=\phi_s^\top\theta+\eta_s$,
\[
V_n\hat\theta_n
=
\sum_{s=1}^n \phi_sY_s
=
\sum_{s=1}^n \phi_s\phi_s^\top\theta+\sum_{s=1}^n \phi_s\eta_s
=
(V_n-I)\theta+Z_n.
\]
Therefore
\[
V_n(\hat\theta_n-\theta)=Z_n-\theta.
\]
Taking the $V_n^{-1}$ norm and applying the triangle inequality,
\[
\|\hat\theta_n-\theta\|_{V_n}
=
\|Z_n-\theta\|_{V_n^{-1}}
\le
\|Z_n\|_{V_n^{-1}}+\|\theta\|_{V_n^{-1}}.
\]
Lemma~\ref{lem:self-normalized-noise} controls the first term. Since $V_n\succeq I$, we have
$V_n^{-1}\preceq I$, hence
\[
\|\theta\|_{V_n^{-1}}
\le
\|\theta\|_2
\le
S.
\]
Combining the two bounds proves the result.
\end{proof}

\begin{lemma}[Predictable masks]
\label{lem:subsequence-self-normalized}
Let $(\cF_s)_{s=0}^T$ be a filtration. Suppose $\phi_s$ is $\cF_{s-1}$-measurable and $\eta_s$ is
conditionally $\sigma$-sub-Gaussian given $\cF_{s-1}$, and $\|\phi_s\|_2\le 1$.
Let $H\in\mathbb S_{++}^d$ be deterministic or, more generally, $\cF_0$-measurable.
Let $m_s\in[0,1]$ be
$\cF_{s-1}$-measurable. Define
\[
\widetilde Z_t := \sum_{s=1}^t m_s \phi_s\eta_s,
\qquad
\widetilde V_t := H+\sum_{s=1}^t m_s^2 \phi_s\phi_s^\top.
\]
Then, with probability at least $1-\delta$,
\[
\|\widetilde Z_T\|_{\widetilde V_T^{-1}}
\le
\sigma\sqrt{2\log\frac{\det(\widetilde V_T)^{1/2}}{\det(H)^{1/2}\delta}}.
\]
\end{lemma}

\begin{proof}
Set $\widetilde\phi_s:=m_s\phi_s$. Since $m_s$ and $\phi_s$ are both $\cF_{s-1}$-measurable,
$\widetilde\phi_s$ is predictable. Moreover $\|\widetilde\phi_s\|_2\le \|\phi_s\|_2\le 1$ because
$m_s\in[0,1]$. Applying Lemma~\ref{lem:self-normalized-noise-base} to the predictable pair
$(\widetilde\phi_s,\eta_s)$ with base matrix $H$ yields the claim.
\end{proof}

\section{Confidence Events and Elimination Lemmas}

This section supplies the entire logical chain behind the main theorem: simultaneous confidence,
combined UCB and LCB validity, survival of the optimal arm, and the stopping-layer regret bound.

For each round let
\[
a_t^\star \in \argmax_{a\in[K]} x_{t,a}^\top\theta_*,
\qquad
\Delta_t(a):=x_{t,a_t^\star}^\top\theta_* - x_{t,a}^\top\theta_*.
\]
Let $\cA_{t,\ell}$ denote the set of arms surviving through layer $\ell$ on round $t$, and let
$\ell_t$ be the stopping layer on round $t$. Define the widths
\[
w_{t,\ell}(a)
:=
U_{t,\ell}^{\min}(a)-L_{t,\ell}^{\max}(a),
\qquad
w_{t,\ell}^{\mathrm{on}}(a)
:=
U_{t,\ell}^{\mathrm{on}}(a)-L_{t,\ell}^{\mathrm{on}}(a),
\]
and similarly
$w_{t,\ell}^{\mathrm{pool}}(a):=U_{t,\ell}^{\mathrm{pool}}(a)-L_{t,\ell}^{\mathrm{pool}}(a)$.

\begin{lemma}[Width comparison]
\label{lem:width-comparison}
For every round, layer, and action,
\[
w_{t,\ell}(a)\le \min\{w_{t,\ell}^{\mathrm{on}}(a),w_{t,\ell}^{\mathrm{pool}}(a)\}
\le w_{t,\ell}^{\mathrm{on}}(a),
\]
and
\[
w_{t,\ell}(a)\le w_{t,\ell}^{\mathrm{pool}}(a)
=2\,\mathrm{rad}^{\mathrm{pool}}_{t,\ell}(x_{t,a}).
\]
\end{lemma}

\begin{proof}
By definition,
\[
w_{t,\ell}(a)
=
\min\{U_{t,\ell}^{\mathrm{on}}(a),U_{t,\ell}^{\mathrm{pool}}(a)\}
-\max\{L_{t,\ell}^{\mathrm{on}}(a),L_{t,\ell}^{\mathrm{pool}}(a)\}.
\]
For any numbers $u_1,u_2,\ell_1,\ell_2$,
\[
\min\{u_1,u_2\}-\max\{\ell_1,\ell_2\}
\le
\min\{u_1-\ell_1,u_2-\ell_2\}.
\]
Applying this with the online and pooled endpoints proves the first display. The second follows from
the explicit form of the pooled UCB--LCB pair.
\end{proof}

Choose confidence levels so that
\[
\sum_{t=1}^T\sum_{\ell=0}^{\lceil \log_2 T\rceil}\delta_{t,\ell}^{\mathrm{on}}
\le \delta_{\mathrm{on}},
\qquad
\delta_{\mathrm{off}}+\delta_{\mathrm{on}}
+\sum_{t=1}^T\sum_{\ell=0}^{\lceil \log_2 T\rceil}\delta_{t,\ell}^{\mathrm{pool}}
\le T^{-2},
\]
abbreviate $\beta_{\mathrm{off}}:=\beta_{\mathrm{off}}(\delta_{\mathrm{off}})$, set
\[
L_T:=\log\!\Bigl(\frac{2KT\log T}{\delta_{\mathrm{on}}}\Bigr).
\]
Also define
\[
\cE_{t,\ell}^{\mathrm{on}}
:=
\left\{
\|b_{t,\ell}-A_{t,\ell}\theta_*\|_{V_{t,\ell}^{-1}}
\le
\beta^{\mathrm{on}}_{t,\ell}-S
\right\}.
\]
Define the events
\[
\cE_{\mathrm{off}}
:=
\left\{\|\hat\theta_{\mathrm{off}}-\theta^\dagger\|_{G_{\mathrm{off}}}\le \beta_{\mathrm{off}}\right\},
\]
\[
\cE_{t,\ell}^{\mathrm{pool}}
:=
\left\{\|b_{t,\ell}-A_{t,\ell}\theta_*\|_{(A_{t,\ell}+G_{\mathrm{off}})^{-1}}\le \gamma_{t,\ell}\right\},
\]
\[
\cE_{\mathrm{on}}
:=
\bigcap_{t=1}^T\bigcap_{\ell=0}^{\lceil \log_2 T\rceil}\cE_{t,\ell}^{\mathrm{on}},
\]
and finally set
\[
\cE
:=
\cE_{\mathrm{on}}
\cap
\cE_{\mathrm{off}}
\cap
\bigcap_{t=1}^T\bigcap_{\ell=0}^{\lceil \log_2 T\rceil}\cE_{t,\ell}^{\mathrm{pool}}.
\]

\begin{lemma}[Simultaneous confidence event]
\label{lem:joint-event}
We have $\Prb(\cE)\ge 1-T^{-2}$.
\end{lemma}

\begin{proof}
By the offline ridge bound,
\[
\Prb(\cE_{\mathrm{off}})\ge 1-\delta_{\mathrm{off}}.
\]
Fix $(t,\ell)$ and define the predictable mask
\[
m_s^{(t,\ell)}:=\mathbf{1}\{s<t,\ s\in\cH_{t,\ell}\},
\qquad s=1,\dots,T.
\]
Because the action set at round $s$ is revealed before the layer assignment and action choice, the
event $\{s\in\cH_{t,\ell}\}$ is determined by information available at time $s-1$ together with the
current contexts; hence $m_s^{(t,\ell)}$ is $\cF_{s-1}$-measurable. Therefore
\[
b_{t,\ell}-A_{t,\ell}\theta_*
=
\sum_{s=1}^{T} m_s^{(t,\ell)}x_{s,a_s}\eta_s,
\qquad
\eta_s:=r_s-x_{s,a_s}^\top\theta_*,
\]
and
\[
A_{t,\ell}+G_{\mathrm{off}}
=
G_{\mathrm{off}}+\sum_{s=1}^{T} m_s^{(t,\ell)}x_{s,a_s}x_{s,a_s}^\top.
\]
Applying Lemma~\ref{lem:subsequence-self-normalized} with
$\phi_s=x_{s,a_s}$, mask $m_s^{(t,\ell)}$, and base matrix $H=G_{\mathrm{off}}$ gives
\[
\|b_{t,\ell}-A_{t,\ell}\theta_*\|_{(A_{t,\ell}+G_{\mathrm{off}})^{-1}}
\le
\sigma\sqrt{
2\log
\frac{
\det(A_{t,\ell}+G_{\mathrm{off}})^{1/2}
}{
\det(G_{\mathrm{off}})^{1/2}\delta_{t,\ell}^{\mathrm{pool}}
}
}
\]
with probability at least $1-\delta_{t,\ell}^{\mathrm{pool}}$, which is exactly the event
$\cE_{t,\ell}^{\mathrm{pool}}$. Thus
\[
\Prb(\cE_{t,\ell}^{\mathrm{pool}})\ge 1-\delta_{t,\ell}^{\mathrm{pool}}.
\]
The same masked bound with confidence level $\delta_{t,\ell}^{\mathrm{on}}$ and base matrix $H=I$
yields
\[
\Prb(\cE_{t,\ell}^{\mathrm{on}})\ge 1-\delta_{t,\ell}^{\mathrm{on}}.
\]
Applying the union bound over all $(t,\ell)$ and using
$\sum_{t,\ell}\delta_{t,\ell}^{\mathrm{on}}\le \delta_{\mathrm{on}}$ gives
\[
\Prb(\cE_{\mathrm{on}})\ge 1-\delta_{\mathrm{on}}.
\]
Applying the union bound over $\cE_{\mathrm{off}}^c$, $\cE_{\mathrm{on}}^c$, and all pooled events
gives
\[
\Prb(\cE^c)
\le
\delta_{\mathrm{off}}+\delta_{\mathrm{on}}
+\sum_{t,\ell}\delta_{t,\ell}^{\mathrm{pool}}
\le T^{-2}.
\]
\end{proof}

\begin{lemma}[Validity of the pooled branch]
\label{lem:pooled-valid}
Fix a round-layer pair $(t,\ell)$. On the event
$\cE_{\mathrm{off}}\cap \cE_{t,\ell}^{\mathrm{pool}}$, every action $a$ satisfies
\[
\left|
x_{t,a}^\top(\hat\theta^{\mathrm{pool}}_{t,\ell}-\theta_*)
\right|
\le
(\gamma_{t,\ell}+\beta_{\mathrm{off}})
\|x_{t,a}\|_{(A_{t,\ell}+G_{\mathrm{off}})^{-1}}
+ \rho\,\psi_{t,\ell}(x_{t,a}).
\]
\end{lemma}

\begin{proof}
Write $A:=A_{t,\ell}$, $G:=G_{\mathrm{off}}$, and $x:=x_{t,a}$. Expanding the pooled estimator,
\[
\hat\theta^{\mathrm{pool}}_{t,\ell}-\theta_*
=(A+G)^{-1}
\Bigl[
(b_{t,\ell}-A\theta_*)
+G(\hat\theta_{\mathrm{off}}-\theta^\dagger)
+G(\theta^\dagger-\theta_*)
\Bigr].
\]
Therefore
\[
\left|x^\top(\hat\theta^{\mathrm{pool}}_{t,\ell}-\theta_*)\right|
\le |T_1|+|T_2|+|T_3|,
\]
where
\[
T_1:=x^\top(A+G)^{-1}(b_{t,\ell}-A\theta_*),
\quad
T_2:=x^\top(A+G)^{-1}G(\hat\theta_{\mathrm{off}}-\theta^\dagger),
\]
\[
T_3:=x^\top(A+G)^{-1}G(\theta^\dagger-\theta_*).
\]
For $T_1$,
\[
|T_1|
=
|x^\top(A+G)^{-1}(b_{t,\ell}-A\theta_*)|
\le
\|x\|_{(A+G)^{-1}}
\|b_{t,\ell}-A\theta_*\|_{(A+G)^{-1}}.
\]
Using $\cE_{t,\ell}^{\mathrm{pool}}$ gives
\[
|T_1|\le \gamma_{t,\ell}\|x\|_{(A+G)^{-1}}.
\]
For $T_2$, write
\[
|T_2|
=
\left|
\bigl(G^{1/2}(A+G)^{-1}x\bigr)^\top
\bigl(G^{1/2}(\hat\theta_{\mathrm{off}}-\theta^\dagger)\bigr)
\right|
\]
and apply Cauchy--Schwarz:
\[
|T_2|
\le
\|x\|_{(A+G)^{-1}G(A+G)^{-1}}
\|\hat\theta_{\mathrm{off}}-\theta^\dagger\|_G.
\]
Because $G\preceq A+G$,
\[
(A+G)^{-1}G(A+G)^{-1}\preceq (A+G)^{-1},
\]
so on $\cE_{\mathrm{off}}$,
\[
|T_2|\le \beta_{\mathrm{off}}\|x\|_{(A+G)^{-1}}.
\]
For $T_3$, let $w:=G(A+G)^{-1}x$. Then
\[
|T_3|=|w^\top(\theta^\dagger-\theta_*)|
\le
\rho\|w\|_{M_{\mathrm{bias}}^{-1}}
=
\rho\,\psi_{t,\ell}(x)
\]
by the directional bias certificate. Summing the three bounds proves the claim.
\end{proof}

\begin{lemma}[Combined UCB and LCB validity]
\label{lem:combined-valid}
On $\cE$, every round-layer-action triple satisfies
\[
x_{t,a}^\top\theta_*
\in
[L_{t,\ell}^{\max}(a),U_{t,\ell}^{\min}(a)].
\]
\end{lemma}

\begin{proof}
On $\cE$, the pooled UCB--LCB pair is valid by Lemma~\ref{lem:pooled-valid}. For the online branch,
\[
\hat\theta^{\mathrm{on}}_{t,\ell}-\theta_*
=
V_{t,\ell}^{-1}(b_{t,\ell}-A_{t,\ell}\theta_*)-V_{t,\ell}^{-1}\theta_*.
\]
Thus
\[
\|\hat\theta^{\mathrm{on}}_{t,\ell}-\theta_*\|_{V_{t,\ell}}
\le
\|b_{t,\ell}-A_{t,\ell}\theta_*\|_{V_{t,\ell}^{-1}}
+ \|\theta_*\|_{V_{t,\ell}^{-1}}
\le
(\beta^{\mathrm{on}}_{t,\ell}-S)+S
=
\beta^{\mathrm{on}}_{t,\ell}
\]
on $\cE_{t,\ell}^{\mathrm{on}}$, since $V_{t,\ell}\succeq I$ and $\|\theta_*\|_2\le S$. Therefore
$|x_{t,a}^\top(\hat\theta^{\mathrm{on}}_{t,\ell}-\theta_*)|
\le \beta^{\mathrm{on}}_{t,\ell}\|x_{t,a}\|_{V_{t,\ell}^{-1}}$ on $\cE$, so the online pair is
also valid. Therefore
\[
x_{t,a}^\top\theta_*
\le
\min\{U_{t,\ell}^{\mathrm{on}}(a),U_{t,\ell}^{\mathrm{pool}}(a)\}
=
U_{t,\ell}^{\min}(a)
\]
and similarly
\[
x_{t,a}^\top\theta_*
\ge
\max\{L_{t,\ell}^{\mathrm{on}}(a),L_{t,\ell}^{\mathrm{pool}}(a)\}
=
L_{t,\ell}^{\max}(a).
\]
\end{proof}

\begin{lemma}[Optimal action survives elimination]
\label{lem:optimal-survives}
On $\cE$, for every round $t$ and layer $\ell$,
\[
a_t^\star \in \cA_{t,\ell}.
\]
\end{lemma}

\begin{proof}
The base case $\ell=0$ is immediate since $\cA_{t,0}=[K]$. Suppose
$a_t^\star\in \cA_{t,\ell}$ and the algorithm descends to layer $\ell+1$. By definition,
\[
m_{t,\ell}:=\max_{a\in \cA_{t,\ell}} L_{t,\ell}^{\max}(a).
\]
Lemma~\ref{lem:combined-valid} implies
\[
L_{t,\ell}^{\max}(a)\le x_{t,a}^\top\theta_* \le x_{t,a_t^\star}^\top\theta_*
\le U_{t,\ell}^{\min}(a_t^\star)
\qquad
\text{for all } a\in \cA_{t,\ell}.
\]
Therefore
\[
U_{t,\ell}^{\min}(a_t^\star)\ge m_{t,\ell},
\]
so the optimal action is not eliminated and remains in $\cA_{t,\ell+1}$.
\end{proof}

\begin{lemma}[Survivor quality after a narrow layer]
\label{lem:survivor-quality}
On $\cE$, fix $t$ and $\ell<\lceil\log_2 T\rceil$. If every arm in $\cA_{t,\ell}$ satisfies
$w_{t,\ell}(a)\le 2^{-\ell}$ and the algorithm descends to layer $\ell+1$, then every
$a\in \cA_{t,\ell+1}$ obeys
\[
\Delta_t(a)\le 2^{1-\ell}.
\]
\end{lemma}

\begin{proof}
Take $a\in \cA_{t,\ell+1}$. By the elimination rule,
\[
U_{t,\ell}^{\min}(a)\ge m_{t,\ell}\ge L_{t,\ell}^{\max}(a_t^\star).
\]
Hence
\begin{align*}
\Delta_t(a)
&=
x_{t,a_t^\star}^\top\theta_* - x_{t,a}^\top\theta_* \\
&=
\bigl(x_{t,a_t^\star}^\top\theta_* - L_{t,\ell}^{\max}(a_t^\star)\bigr)
+\bigl(L_{t,\ell}^{\max}(a_t^\star)-U_{t,\ell}^{\min}(a)\bigr)
+\bigl(U_{t,\ell}^{\min}(a)-x_{t,a}^\top\theta_*\bigr).
\end{align*}
The middle term is nonpositive, and the first and third terms are each at most $2^{-\ell}$ by
Lemma~\ref{lem:combined-valid} and the width assumption. Thus
\[
\Delta_t(a)\le 2^{-\ell}+2^{-\ell}=2^{1-\ell}.
\]
\end{proof}

\begin{corollary}[Layer invariant]
\label{cor:layer-invariant}
On $\cE$, every round $t$ and every layer $\ell\ge 1$ satisfy
\[
\Delta_t(a)\le 2^{2-\ell}
\qquad
\text{for all } a\in \cA_{t,\ell}.
\]
\end{corollary}

\begin{proof}
On $\cE$, apply Lemma~\ref{lem:survivor-quality} at layer $\ell-1$.
\end{proof}

\begin{lemma}[Stopping-layer regret bound]
\label{lem:stopping-regret}
On $\cE$, every round $t$ satisfies
\[
\Delta_t(a_t)
\le
4\min\{1,w_{t,\ell_t}(a_t)\}
\le
8\min\{1,\mathrm{rad}_{t,\ell_t}^{\mathrm{pool}}(x_{t,a_t})\}.
\]
\end{lemma}

\begin{proof}
We consider the stopping layer $\ell_t$.

\paragraph{Case 1: $\ell_t=\lceil\log_2 T\rceil$.}
The played arm maximizes $U_{t,\ell_t}^{\min}$ over $\cA_{t,\ell_t}$, and
$a_t^\star\in \cA_{t,\ell_t}$ by Lemma~\ref{lem:optimal-survives}. Therefore
\[
U_{t,\ell_t}^{\min}(a_t)\ge U_{t,\ell_t}^{\min}(a_t^\star)\ge x_{t,a_t^\star}^\top\theta_*.
\]
Using Lemma~\ref{lem:combined-valid},
\[
\Delta_t(a_t)
\le
U_{t,\ell_t}^{\min}(a_t)-L_{t,\ell_t}^{\max}(a_t)
=
w_{t,\ell_t}(a_t).
\]
Also, rewards lie in $[0,1]$, so $\Delta_t(a_t)\le 1$. Hence
\[
\Delta_t(a_t)
\le
\min\{1,w_{t,\ell_t}(a_t)\}
\le
4\min\{1,w_{t,\ell_t}(a_t)\}.
\]

\paragraph{Case 2: $\ell_t<\lceil\log_2 T\rceil$.}
If $\ell_t=0$, then $\Delta_t(a_t)\le 1$ because rewards are bounded in $[0,1]$, and
$1\le 4\min\{1,w_{t,0}(a_t)\}$ since the algorithm stops only when $w_{t,0}(a_t)>1$ or the terminal
layer is reached.

If $\ell_t\ge 1$, Corollary~\ref{cor:layer-invariant} gives
\[
\Delta_t(a_t)\le 2^{2-\ell_t}.
\]
But the stopping rule at a nonterminal layer implies $w_{t,\ell_t}(a_t)>2^{-\ell_t}$, so
\[
4\,w_{t,\ell_t}(a_t)> 4\cdot 2^{-\ell_t}=2^{2-\ell_t}.
\]
Hence $\Delta_t(a_t)\le 4\,w_{t,\ell_t}(a_t)$.

Finally, Lemma~\ref{lem:width-comparison} implies
\[
w_{t,\ell_t}(a_t)\le 2\,\mathrm{rad}_{t,\ell_t}^{\mathrm{pool}}(x_{t,a_t}),
\]
so
\[
4\min\{1,w_{t,\ell_t}(a_t)\}
\le
4\min\{1,2\,\mathrm{rad}_{t,\ell_t}^{\mathrm{pool}}(x_{t,a_t})\}
\le
8\min\{1,\mathrm{rad}_{t,\ell_t}^{\mathrm{pool}}(x_{t,a_t})\},
\]
which yields the second bound.
\end{proof}

\begin{corollary}[Mixed-branch master bound]
\label{cor:mixed-branch-master}
On $\cE$, every round $t$ satisfies
\[
\Delta_t(a_t)
\le
8\min\left\{
1,\,
\mathrm{rad}^{\mathrm{on}}_{t,\ell_t}(x_{t,a_t}),\,
\mathrm{rad}^{\mathrm{pool}}_{t,\ell_t}(x_{t,a_t})
\right\}.
\]
Consequently,
\[
R_T
\le
8\sum_{t=1}^T
\min\left\{
1,\,
\mathrm{rad}^{\mathrm{on}}_{t,\ell_t}(x_{t,a_t}),\,
\mathrm{rad}^{\mathrm{pool}}_{t,\ell_t}(x_{t,a_t})
\right\}.
\]
\end{corollary}

\begin{proof}
By Lemma~\ref{lem:stopping-regret},
\[
\Delta_t(a_t)\le 4\min\{1,w_{t,\ell_t}(a_t)\}.
\]
Lemma~\ref{lem:width-comparison} and the explicit online and pooled widths give
\[
w_{t,\ell_t}(a_t)
\le
\min\{w_{t,\ell_t}^{\mathrm{on}}(a_t),w_{t,\ell_t}^{\mathrm{pool}}(a_t)\}
=
2\min\!\left\{
\mathrm{rad}^{\mathrm{on}}_{t,\ell_t}(x_{t,a_t}),
\mathrm{rad}^{\mathrm{pool}}_{t,\ell_t}(x_{t,a_t})
\right\}.
\]
Therefore, if
\[
m_t
:=
\min\!\left\{
\mathrm{rad}^{\mathrm{on}}_{t,\ell_t}(x_{t,a_t}),
\mathrm{rad}^{\mathrm{pool}}_{t,\ell_t}(x_{t,a_t})
\right\},
\]
then
\[
\Delta_t(a_t)
\le
4\min\{1,2m_t\}
\le
8\min\{1,m_t\}
=
8\min\!\left\{
1,\,
\mathrm{rad}^{\mathrm{on}}_{t,\ell_t}(x_{t,a_t}),\,
\mathrm{rad}^{\mathrm{pool}}_{t,\ell_t}(x_{t,a_t})
\right\}.
\]
This proves the first display, and summing over $t$ yields the second.
\end{proof}

\section{Roadmap for the Online Fallback}
\label{app:suplin-roadmap}

The online branch follows the classical SupLinUCB logic. We spell out the two standard ingredients
needed to transfer the fallback bound to our combined UCB/LCB wrapper.

\begin{lemma}[Online width trigger]
\label{lem:online-trigger}
If round $t$ stops at a nonterminal layer $\ell_t$, then
\[
w_{t,\ell_t}^{\mathrm{on}}(a_t)>2^{-\ell_t}.
\]
\end{lemma}

\begin{proof}
The stopping rule gives $w_{t,\ell_t}(a_t)>2^{-\ell_t}$, and
$w_{t,\ell_t}(a_t)\le w_{t,\ell_t}^{\mathrm{on}}(a_t)$ by Lemma~\ref{lem:width-comparison}.
\end{proof}

\begin{lemma}[Layerwise online potential]
\label{lem:online-potential}
Let
\[
Q_\ell:=\{t\in[T]:\ell_t=\ell\}
\]
and
\[
V_{\ell}^{\mathrm{fin}}
:=
I+\sum_{t\in Q_\ell} x_{t,a_t}x_{t,a_t}^\top.
\]
Then
\[
\sum_{t\in Q_\ell}\min\left\{1,\|x_{t,a_t}\|_{V_{t,\ell}^{-1}}^2\right\}
\le
2\log\frac{\det(V_{\ell}^{\mathrm{fin}})}{\det(I)}
\le
2d\log(1+T).
\]
\end{lemma}

\begin{proof}
Enumerate $Q_\ell$ as $t_1<\cdots<t_{|Q_\ell|}$ and define
\[
W_{\ell,0}:=I,
\qquad
W_{\ell,j}:=I+\sum_{i=1}^j x_{t_i,a_{t_i}}x_{t_i,a_{t_i}}^\top.
\]
Then $W_{\ell,j-1}=V_{t_j,\ell}$. By the matrix-determinant lemma,
\[
\det(W_{\ell,j})
=
\det(W_{\ell,j-1})\left(1+\|x_{t_j,a_{t_j}}\|_{W_{\ell,j-1}^{-1}}^2\right).
\]
Summing logarithms gives
\[
\log\frac{\det(V_{\ell}^{\mathrm{fin}})}{\det(I)}
=
\sum_{j=1}^{|Q_\ell|}
\log\left(1+\|x_{t_j,a_{t_j}}\|_{W_{\ell,j-1}^{-1}}^2\right).
\]
Using $\log(1+u)\ge \frac12\min\{u,1\}$ proves the first inequality. The second follows from
AM--GM:
\[
\det(V_{\ell}^{\mathrm{fin}})
\le
\left(1+\frac{T}{d}\right)^d.
\]
\end{proof}

\begin{proposition}[Classical online fallback]
\label{prop:online-fallback}
On $\cE$, the online branch yields
\[
R_T\le C_{\mathrm{SL}}\sqrt{Td\,L_T^3}.
\]
\end{proposition}

\begin{proof}
The standard SupLinUCB counting argument uses only the following ingredients:
\begin{enumerate}
\item the optimal action survives elimination at every layer;
\item every arm surviving a narrow layer is near-optimal;
\item stopping at a nonterminal layer shows a large online width;
\item the layerwise online widths satisfy an elliptical-potential bound.
\end{enumerate}
In our wrapper these are supplied, respectively, by
Lemmas~\ref{lem:optimal-survives}, \ref{lem:survivor-quality},
\ref{lem:online-trigger}, and \ref{lem:online-potential}. The proof of
Theorem~1 in \cite{chu2011} is purely pathwise once those four ingredients hold, so the same dyadic
summation yields $R_T\le C_{\mathrm{SL}}\sqrt{Td\,L_T^3}$ on $\cE$.
\end{proof}

\section{Proof of Theorem~\ref{thm:main}}
\label{app:proof-main}

We now prove the pooled part of the theorem in full detail.

For each layer let
\[
Q_\ell:=\{t\in[T]:\ell_t=\ell\},
\qquad
A_\ell^{\mathrm{fin}}
:=
\sum_{t\in Q_\ell}x_{t,a_t}x_{t,a_t}^\top.
\]
Define
\[
\Lambda_\ell
:=
\log\frac{\det(A_\ell^{\mathrm{fin}}+G_{\mathrm{off}})}{\det(G_{\mathrm{off}})},
\qquad
\widetilde\Lambda_T
:=
\sum_{\ell=0}^{\lceil\log_2 T\rceil}\Lambda_\ell,
\qquad
\widetilde\Psi_T
:=
\sum_{t=1}^T \psi_{t,\ell_t}(x_{t,a_t}).
\]
Also set
\[
\delta_{\ell,\min}^{\mathrm{pool}}
:=
\min_{t\le T}\delta_{t,\ell}^{\mathrm{pool}},
\qquad
\gamma_\ell^{\mathrm{fin}}
:=
\sigma\sqrt{\Lambda_\ell+2\log(1/\delta_{\ell,\min}^{\mathrm{pool}})},
\]
\[
\Lambda_{\max}
:=
\max_{0\le \ell\le \lceil\log_2 T\rceil}\Lambda_\ell,
\qquad
\delta_{\min}^{\mathrm{pool}}
:=
\min_{t,\ell}\delta_{t,\ell}^{\mathrm{pool}},
\qquad
\Gamma_{\mathrm{pool}}
:=
\beta_{\mathrm{off}}
+
\sigma\sqrt{\Lambda_{\max}+2\log(1/\delta_{\min}^{\mathrm{pool}})},
\]
and
\[
\Gamma_{\mathrm{spec}}
:=
\beta_{\mathrm{off}}
+
\sigma\sqrt{\Phi_{G_{\mathrm{off}}}(T)+2\log(1/\delta_{\min}^{\mathrm{pool}})}.
\]

\begin{lemma}[Layerwise pooled elliptical potential]
\label{lem:pooled-potential}
For every layer $\ell$,
\[
\sum_{t\in Q_\ell}\|x_{t,a_t}\|_{(A_{t,\ell}+G_{\mathrm{off}})^{-1}}^2
\le
2\log\frac{\det(A_\ell^{\mathrm{fin}}+G_{\mathrm{off}})}{\det(G_{\mathrm{off}})}.
\]
\end{lemma}

\begin{proof}
Enumerate $Q_\ell$ as $t_1<\cdots<t_{|Q_\ell|}$ and define
\[
W_{\ell,0}:=G_{\mathrm{off}},
\qquad
W_{\ell,j}:=G_{\mathrm{off}}+\sum_{i=1}^j x_{t_i,a_{t_i}}x_{t_i,a_{t_i}}^\top.
\]
Then $W_{\ell,j-1}=A_{t_j,\ell}+G_{\mathrm{off}}$ and
$W_{\ell,|Q_\ell|}=A_\ell^{\mathrm{fin}}+G_{\mathrm{off}}$. The matrix-determinant lemma yields
\[
\det(W_{\ell,j})
=
\det(W_{\ell,j-1})
\left(
1+\|x_{t_j,a_{t_j}}\|_{W_{\ell,j-1}^{-1}}^2
\right).
\]
Taking logarithms and summing,
\[
\log\frac{\det(A_\ell^{\mathrm{fin}}+G_{\mathrm{off}})}{\det(G_{\mathrm{off}})}
=
\sum_{j=1}^{|Q_\ell|}
\log\left(1+\|x_{t_j,a_{t_j}}\|_{W_{\ell,j-1}^{-1}}^2\right).
\]
Since $G_{\mathrm{off}}\succeq I$ and $W_{\ell,j-1}\succeq G_{\mathrm{off}}$, we have
$W_{\ell,j-1}^{-1}\preceq I$, hence for $\|x_{t_j,a_{t_j}}\|_2\le 1$,
\[
\|x_{t_j,a_{t_j}}\|_{W_{\ell,j-1}^{-1}}^2\le 1.
\]
Therefore $\log(1+u)\ge u/2$ on this range, and the claim follows.
\end{proof}

\begin{corollary}[Pooled variance accumulation]
\label{cor:pooled-variance}
We have
\[
\sum_{t=1}^T \|x_{t,a_t}\|_{(A_{t,\ell_t}+G_{\mathrm{off}})^{-1}}
\le
\sqrt{2T\,\widetilde\Lambda_T}.
\]
\end{corollary}

\begin{proof}
For each layer,
\[
\sum_{t\in Q_\ell}\|x_{t,a_t}\|_{(A_{t,\ell}+G_{\mathrm{off}})^{-1}}
\le
\sqrt{|Q_\ell|}
\left(
\sum_{t\in Q_\ell}\|x_{t,a_t}\|_{(A_{t,\ell}+G_{\mathrm{off}})^{-1}}^2
\right)^{1/2}
\]
by Cauchy--Schwarz. Apply Lemma~\ref{lem:pooled-potential}, then sum over $\ell$ and apply
Cauchy--Schwarz again across layers.
\end{proof}

\begin{proof}[Proof of Theorem~\ref{thm:main}]
On $\cE$, Proposition~\ref{prop:online-fallback} gives the fallback term
$C_{\mathrm{SL}}\sqrt{Td\,L_T^3}$. It therefore remains to prove the pooled term.

By Lemma~\ref{lem:stopping-regret},
\[
R_T
\le
8\sum_{t=1}^T \min\left\{1,\mathrm{rad}_{t,\ell_t}^{\mathrm{pool}}(x_{t,a_t})\right\}
\le
8\sum_{t=1}^T \mathrm{rad}_{t,\ell_t}^{\mathrm{pool}}(x_{t,a_t}),
\]
since $\min\{1,r\}\le r$ for $r\ge 0$. Expanding the pooled radius,
\[
R_T
\le
8\sum_{\ell=0}^{\lceil\log_2 T\rceil}\sum_{t\in Q_\ell}
\bigl(\gamma_{t,\ell_t}+\beta_{\mathrm{off}}\bigr)
\|x_{t,a_t}\|_{(A_{t,\ell_t}+G_{\mathrm{off}})^{-1}}
+8\rho\sum_{t=1}^T \psi_{t,\ell_t}(x_{t,a_t}).
\]
For every $t\in Q_\ell$, we have $\delta_{t,\ell}^{\mathrm{pool}}\ge \delta_{\ell,\min}^{\mathrm{pool}}$
and $A_{t,\ell}\preceq A_\ell^{\mathrm{fin}}$, so
\[
\det(A_{t,\ell}+G_{\mathrm{off}})\le \det(A_\ell^{\mathrm{fin}}+G_{\mathrm{off}})
\]
and hence $\gamma_{t,\ell}\le \gamma_\ell^{\mathrm{fin}}$.
Therefore
\[
R_T
\le
8\sum_{\ell=0}^{\lceil\log_2 T\rceil}
\bigl(\beta_{\mathrm{off}}+\gamma_\ell^{\mathrm{fin}}\bigr)
\sum_{t\in Q_\ell}\|x_{t,a_t}\|_{(A_{t,\ell}+G_{\mathrm{off}})^{-1}}
+ 8\rho\,\widetilde\Psi_T.
\]
For each layer, Cauchy--Schwarz and Lemma~\ref{lem:pooled-potential} give
\[
\sum_{t\in Q_\ell}\|x_{t,a_t}\|_{(A_{t,\ell}+G_{\mathrm{off}})^{-1}}
\le
\sqrt{|Q_\ell|}
\left(
\sum_{t\in Q_\ell}\|x_{t,a_t}\|_{(A_{t,\ell}+G_{\mathrm{off}})^{-1}}^2
\right)^{1/2}
\le
\sqrt{2|Q_\ell|\Lambda_\ell}.
\]
Substituting this into the previous display proves the theorem's first pooled bound:
\[
R_T
\le
8\sum_{\ell=0}^{\lceil\log_2 T\rceil}
\bigl(\beta_{\mathrm{off}}+\gamma_\ell^{\mathrm{fin}}\bigr)\sqrt{2|Q_\ell|\Lambda_\ell}
+ 8\rho\,\widetilde\Psi_T.
\]
For the coarser simplification, note that $\delta_{\ell,\min}^{\mathrm{pool}}\ge\delta_{\min}^{\mathrm{pool}}$
and $\Lambda_\ell\le \Lambda_{\max}$, so every $\gamma_\ell^{\mathrm{fin}}$ is at most
$\Gamma_{\mathrm{pool}}-\beta_{\mathrm{off}}$. Hence
\[
\sum_{\ell=0}^{\lceil\log_2 T\rceil}
\bigl(\beta_{\mathrm{off}}+\gamma_\ell^{\mathrm{fin}}\bigr)\sqrt{2|Q_\ell|\Lambda_\ell}
\le
\Gamma_{\mathrm{pool}}
\sum_{\ell=0}^{\lceil\log_2 T\rceil}\sqrt{2|Q_\ell|\Lambda_\ell}.
\]
Applying Cauchy--Schwarz across layers and using $\sum_\ell |Q_\ell|=T$ yields
\[
\sum_{\ell=0}^{\lceil\log_2 T\rceil}\sqrt{2|Q_\ell|\Lambda_\ell}
\le
\sqrt{2T\,\widetilde\Lambda_T},
\]
which proves the theorem's second pooled bound. Both the fallback term and the pooled terms hold on
the same event $\cE$, so taking the minimum of the two completes the proof.
\end{proof}

\section{Proof of Theorem~\ref{thm:spectral-pooled-envelope}}

\begin{proof}
Write $L_{\max}=\lceil\log_2 T\rceil+1$, matching the number of SupLinUCB layers indexed by
$\ell\in\{0,\ldots,\lceil\log_2 T\rceil\}$.
Fix $\ell$ and diagonalize $G_{\mathrm{off}}=U\,\mathrm{diag}(g_1,\ldots,g_d)\,U^\top$ with $U$ orthogonal.
Let $\widetilde A_\ell:=U^\top A_\ell^{\mathrm{fin}}U$.
Hadamard's inequality gives
\[
\det(A_\ell^{\mathrm{fin}}+G_{\mathrm{off}})
=
\det(\widetilde A_\ell+\mathrm{diag}(g_1,\ldots,g_d))
\le
\prod_{j=1}^d\bigl(g_j+(\widetilde A_\ell)_{jj}\bigr).
\]
Therefore
\[
\Lambda_\ell
=
\log\frac{\det(A_\ell^{\mathrm{fin}}+G_{\mathrm{off}})}{\det(G_{\mathrm{off}})}
\le
\sum_{j=1}^d
\log\left(1+\frac{(\widetilde A_\ell)_{jj}}{g_j}\right).
\]
Because $\sum_j (\widetilde A_\ell)_{jj}=\mathrm{tr}(A_\ell^{\mathrm{fin}})\le |Q_\ell|$ and each played feature has
$\|x_{t,a_t}\|_2\le 1$, Definition~\ref{def:waterfill-info} yields
$\Lambda_\ell\le \Phi_{G_{\mathrm{off}}}(|Q_\ell|)$.
Summing over layers,
\[
\widetilde\Lambda_T
\le
\sum_{\ell=0}^{\lceil\log_2 T\rceil}\Phi_{G_{\mathrm{off}}}(|Q_\ell|).
\]
The map $B\mapsto \Phi_G(B)$ is nondecreasing and concave in $B$: it is the value function of maximizing a sum of
concave coordinate functions $\log(1+a_j/g_j)$ subject to $a_j\ge 0$ and $\sum_j a_j\le B$.
Hence Jensen's inequality implies
\[
\frac{1}{L_{\max}}\sum_{\ell=0}^{\lceil\log_2 T\rceil}\Phi_{G_{\mathrm{off}}}(|Q_\ell|)
\le
\Phi_{G_{\mathrm{off}}}\!\left(\frac{1}{L_{\max}}\sum_{\ell=0}^{\lceil\log_2 T\rceil}|Q_\ell|\right)
=
\Phi_{G_{\mathrm{off}}}\!\left(\frac{T}{L_{\max}}\right),
\]
since $\sum_\ell |Q_\ell|=T$.
Multiplying by $L_{\max}$ proves the spectral envelope bound on $\widetilde\Lambda_T$.

Define
\[
c_{\mathrm{align}}
:=
\lambda_{\max}\!\left(G_{\mathrm{off}}^{1/2}M_{\mathrm{bias}}^{-1}G_{\mathrm{off}}^{1/2}\right).
\]
Then $G_{\mathrm{off}}^{1/2}M_{\mathrm{bias}}^{-1}G_{\mathrm{off}}^{1/2}\preceq c_{\mathrm{align}}I$, hence, conjugating by
$G_{\mathrm{off}}^{1/2}$,
\[
G_{\mathrm{off}}M_{\mathrm{bias}}^{-1}G_{\mathrm{off}}\preceq c_{\mathrm{align}}G_{\mathrm{off}}.
\]
For the alignment bound on $\widetilde\Psi_T$, repeat the Cauchy--Schwarz argument used previously for the pooled regime:
write $A:=A_{t,\ell}$, $G:=G_{\mathrm{off}}$, and $M:=M_{\mathrm{bias}}$.
Under $GM^{-1}G\preceq c_{\mathrm{align}}G$,
\[
\psi_{t,\ell}(x)^2
\le
c_{\mathrm{align}}\,\|x\|_{(A+G)^{-1}}^2,
\]
then sum along the trajectory and invoke Lemma~\ref{lem:pooled-potential} layer by layer to obtain
$\sum_{t=1}^T \psi_{t,\ell_t}(x_{t,a_t})^2\le 2c_{\mathrm{align}}\widetilde\Lambda_T$, hence
$\widetilde\Psi_T\le\sqrt{2c_{\mathrm{align}}T\,\widetilde\Lambda_T}$.

Finally, substitute these inequalities into the pooled branch bound from Theorem~\ref{thm:main}:
\[
R_T
\le
8\Gamma_{\mathrm{pool}}\sqrt{2T\,\widetilde\Lambda_T}
+
8\rho\sqrt{2c_{\mathrm{align}}T\,\widetilde\Lambda_T},
\]
so
\[
R_T
\le
8\bigl(\Gamma_{\mathrm{pool}}+\rho\sqrt{c_{\mathrm{align}}}\bigr)
\sqrt{2T\,\widetilde\Lambda_T}.
\]
Since $\Lambda_\ell\le \Phi_{G_{\mathrm{off}}}(|Q_\ell|)\le \Phi_{G_{\mathrm{off}}}(T)$ for every
layer, we also have $\Lambda_{\max}\le \Phi_{G_{\mathrm{off}}}(T)$ and therefore
$\Gamma_{\mathrm{pool}}\le \Gamma_{\mathrm{spec}}$.
Substituting this and the bound
\(
\widetilde\Lambda_T\le L_{\max}\Phi_{G_{\mathrm{off}}}(T/L_{\max})
\)
proves the displayed theorem.
The online fallback is unchanged from Theorem~\ref{thm:main}.
\end{proof}

\begin{proof}[Proof of Corollary~\ref{cor:pooled-better-regime}]
The first display is exactly the intermediate inequality obtained at the end of the proof of
Theorem~\ref{thm:spectral-pooled-envelope}. Squaring the pooled term yields
\[
\left(
8\bigl(\Gamma_{\mathrm{pool}}+\rho\sqrt{c_{\mathrm{align}}}\bigr)
\sqrt{2T\,\widetilde\Lambda_T}
\right)^2
=
128\,T\,\widetilde\Lambda_T
\bigl(\Gamma_{\mathrm{pool}}+\rho\sqrt{c_{\mathrm{align}}}\bigr)^2.
\]
Hence the stated condition implies that the pooled term is at most
$C_{\mathrm{SL}}\sqrt{Td\,L_T^3}$, i.e.\ no larger than the online fallback.
\end{proof}

\section{Proof of Theorem~\ref{thm:lower}}
\label{app:proof-lower}

\begin{proof}
It is enough to prove the theorem with a fixed absolute constant $m_0=32$. Fix any
(possibly randomized) policy $\pi$, and fix any subset $S\subset[d]$ with
\[
m:=|S|\ge 32.
\]
Write
\[
N:=N(S)=\sum_{i\in S} n_i,
\qquad
\underline v:=v(S)=\min_{i\in S}\frac{\rho}{\sqrt{b_i}}.
\]

\paragraph{Case 1: $N=0$.}
Since $n_i=0$ for every $i\in S$, the offline sample carries no information about the coordinates in
$S$. Consider the baseline instance
\[
P_0:\qquad \theta^*=\theta^\dagger=0,
\]
and, for each $j\in S$, the alternative
\[
P_j:\qquad
\theta_j^*=\theta_j^\dagger=\Delta,
\qquad
\theta_i^*=\theta_i^\dagger=0\quad\text{for }i\neq j.
\]
Each $P_j$ belongs to $\mathcal H_{\mathrm{diag}}(G_{\mathrm{off}},M_{\mathrm{bias}},\rho)$ because
the bias is zero.

Let
\[
N_j:=\sum_{t=1}^T \mathbf{1}\{a_t=j\}.
\]
Under $P_0$, there exists $j^\star\in S$ with
\[
\E_{P_0}[N_{j^\star}]\le \frac{T}{m},
\]
since $\sum_{j\in S}\E_{P_0}[N_j]\le T$. Define
\[
E:=\{N_{j^\star}>T/2\}.
\]
Then Markov's inequality gives
\[
\Pr_{P_0}(E)\le \frac{\E_{P_0}[N_{j^\star}]}{T/2}\le \frac{2}{m}\le \frac{1}{16}.
\]
The offline laws of $P_0$ and $P_{j^\star}$ are identical, and the online KL is
\[
\mathrm{KL}(P_0,P_{j^\star})
=
\frac{\Delta^2}{2}\E_{P_0}[N_{j^\star}]
\le
\frac{\Delta^2T}{2m}.
\]
Choose
\[
\Delta:=\frac{1}{8}\sqrt{\frac{m}{T}}.
\]
Then $\mathrm{KL}(P_0,P_{j^\star})\le 1/128<1/100$. By Bretagnolle--Huber,
\[
\Pr_{P_{j^\star}}(E^c)
\ge
\frac12 e^{-1/100}-\frac{1}{16}
> c_0
\]
for a universal constant $c_0>0$.

Under $P_{j^\star}$, arm $j^\star$ is uniquely optimal with gap $\Delta$. On $E^c$, the optimal arm
is pulled at most $T/2$ times, so at least $T/2$ rounds incur gap $\Delta$. Hence
\[
\E_{P_{j^\star}}[R_T]
\ge
\frac{\Delta T}{2}\Pr_{P_{j^\star}}(E^c)
\ge
c_1\sqrt{mT}.
\]
When $N=0$,
\[
\Psi(S)=\sqrt{\frac{mT^2}{T+N}}+\min\left\{\frac{NT}{T+N}\underline v,\,
\sqrt{\frac{mNT}{T+N}}\right\}=\sqrt{mT},
\]
so
\[
\sup_{I\in \mathcal H_{\mathrm{diag}}(G_{\mathrm{off}},M_{\mathrm{bias}},\rho)}
\E_I[R_T]
\ge
c_1\Psi(S).
\]

\paragraph{Case 2: $N>0$.}
Define the clipped bias size
\[
\bar v:=\min\left\{\underline v,\;\frac{1}{32}\sqrt{\frac{m(T+N)}{NT}}\right\}.
\]
Consider the baseline instance
\[
P_0:\qquad \theta^*=\theta^\dagger=0,
\]
and, for each $j\in S$, the alternative
\[
P_j:\qquad
\theta_j^*=\mu,
\qquad
\theta_j^\dagger=\mu-\bar v,
\qquad
\theta_i^*=\theta_i^\dagger=0\quad\text{for }i\neq j,
\]
where
\[
\mu:=\frac{N}{N+T}\bar v+\frac{1}{64}\sqrt{\frac{m}{N+T}}.
\]
Since $\bar v\le \underline v\le \rho/\sqrt{b_j}$ for every $j\in S$, each $P_j$ belongs to
$\mathcal H_{\mathrm{diag}}(G_{\mathrm{off}},M_{\mathrm{bias}},\rho)$.

Again write
\[
N_j:=\sum_{t=1}^T \mathbf{1}\{a_t=j\}.
\]
Under $P_0$, define
\[
B_1:=\{j\in S:\; n_j>4N/m\},
\qquad
B_2:=\{j\in S:\; \E_{P_0}[N_j]>4T/m\}.
\]
Since
\[
\sum_{j\in S} n_j=N
\qquad\text{and}\qquad
\sum_{j\in S}\E_{P_0}[N_j]\le T,
\]
we have
\[
|B_1|\le \frac{m}{4},
\qquad
|B_2|\le \frac{m}{4}.
\]
Therefore there exists an arm $j^\star\in S\setminus(B_1\cup B_2)$ such that
\begin{equation}
\label{eq:diag-lb-jstar}
n_{j^\star}\le \frac{4N}{m},
\qquad
\E_{P_0}[N_{j^\star}]\le \frac{4T}{m}.
\end{equation}

Only arm $j^\star$ differs between $P_0$ and $P_{j^\star}$. Since offline and online rewards are
Gaussian with unit variance,
\[
\mathrm{KL}(P_0,P_{j^\star})
=
\frac12\left(
n_{j^\star}(\mu-\bar v)^2+\E_{P_0}[N_{j^\star}]\mu^2
\right).
\]
Using \eqref{eq:diag-lb-jstar},
\[
\mathrm{KL}(P_0,P_{j^\star})
\le
\frac12\left(
\frac{4N}{m}(\mu-\bar v)^2+\frac{4T}{m}\mu^2
\right).
\]
Set
\[
a:=\frac{4N}{m},
\qquad
b:=\frac{4T}{m}.
\]
Then
\[
\mu=\frac{a}{a+b}\bar v+\frac{1}{32\sqrt{a+b}},
\]
because \(a+b=4(N+T)/m\). A direct expansion gives
\[
a(\mu-\bar v)^2+b\mu^2=\frac{ab}{a+b}\bar v^2+\frac{1}{32^2}.
\]
Also,
\[
\frac{ab}{a+b}=\frac{4NT}{m(N+T)}
\]
and, by the definition of $\bar v$,
\[
\bar v^2\le \frac{1}{32^2}\frac{m(T+N)}{NT}.
\]
Hence
\[
\frac{ab}{a+b}\bar v^2\le \frac{4}{32^2},
\]
so
\[
\mathrm{KL}(P_0,P_{j^\star})
\le
\frac12\left(\frac{4}{32^2}+\frac{1}{32^2}\right)
=
\frac{5}{2048}
<
\frac{1}{100}.
\tag{$\ast$}
\]

Define
\[
E:=\{N_{j^\star}>T/2\}.
\]
By Markov's inequality and \eqref{eq:diag-lb-jstar},
\[
\Pr_{P_0}(E)\le \frac{\E_{P_0}[N_{j^\star}]}{T/2}\le \frac{8}{m}\le \frac14,
\]
since $m\ge 32$. Bretagnolle--Huber and \((\ast)\) give
\[
\Pr_{P_{j^\star}}(E^c)
\ge
\frac12 e^{-1/100}-\frac14
> c_2
\]
for a universal constant $c_2>0$.

In instance $P_{j^\star}$, arm $j^\star$ is uniquely optimal with gap $\mu$. On $E^c$, the optimal
arm is pulled at most $T/2$ times, so at least $T/2$ rounds incur gap $\mu$. Therefore
\[
\E_{P_{j^\star}}[R_T]
\ge
\frac{\mu T}{2}\Pr_{P_{j^\star}}(E^c)
\ge
c_3\,\mu T.
\]
Finally,
\[
\mu T
=
\frac{NT}{N+T}\bar v+\frac{1}{64}\sqrt{\frac{mT^2}{N+T}},
\]
and by the definition of $\bar v$,
\[
\frac{NT}{N+T}\bar v
=
\min\left\{
\frac{NT}{N+T}\,\underline v,\;
\frac{1}{32}\sqrt{\frac{mNT}{N+T}}
\right\}.
\]
Hence
\[
\sup_{I\in \mathcal H_{\mathrm{diag}}(G_{\mathrm{off}},M_{\mathrm{bias}},\rho)}
\E_I[R_T]
\ge
c_4\left[
\sqrt{\frac{mT^2}{T+N}}
+
\min\left\{
\frac{NT}{T+N}\,\underline v,\;
\sqrt{\frac{mNT}{T+N}}
\right\}
\right]
=
c_4\,\Psi(S).
\]

Both cases show that, for every subset $S\subset[d]$ with $|S|\ge 32$,
\[
\sup_{I\in \mathcal H_{\mathrm{diag}}(G_{\mathrm{off}},M_{\mathrm{bias}},\rho)}
\E_I[R_T]
\ge
c_5\,\Psi(S)
\ge
c_5\min\{\sqrt{dT},\Psi(S)\},
\]
since $\min\{\sqrt{dT},\Psi(S)\}\le \Psi(S)$. Taking the maximum over all such subsets proves the
theorem with \(m_0=32\).
\end{proof}

\section{Data-driven directional bias certificates: definitions and proofs (Section~\ref{sec:epoch-wise-ell-minucb})}
\label{app:epoch-learned-proofs}

\subsection{Definitions and Algorithm}
\label{app:epoch-dd-details}

In this section, we give the detailed definitions and proofs in
Section~\ref{sec:epoch-wise-ell-minucb}. We first define the epoch-wise pooled width and the full
algorithm, then record the geometric identity behind the learned certificate, and finally prove the
validity and regret guarantees together with the doubling-epoch refinement. The pseudocode of the algorithms in given in Algorithm~\ref{alg:epoch-ell-minucb}.

On rounds $t\in\{\tau_k,\ldots,\tau_{k+1}-1\}$ one replaces
$\rho\,\psi_{t,\ell}$ of the main text by $\widehat\rho_k\,\widehat\psi_{k,t,\ell}$ with
\begin{equation}
\label{eq:epoch-psi}
\widehat\psi_{k,t,\ell}(x)
:=
\left\|
\widehat M_{k,\mathrm{bias}}^{-1/2}G_{\mathrm{off}}(A_{t,\ell}+G_{\mathrm{off}})^{-1}x
\right\|_2
\end{equation}
and
\begin{equation}
\label{eq:epoch-pool-rad}
\widehat{\mathrm{rad}}^{\mathrm{pool,ep}}_{t,\ell}(x)
:=
\bigl(\gamma_{t,\ell}+\beta_{\mathrm{off}}\bigr)\|x\|_{(A_{t,\ell}+G_{\mathrm{off}})^{-1}}
+\widehat\rho_k\,\widehat\psi_{k,t,\ell}(x).
\end{equation}

\begin{algorithm}[t]
\caption{Data-Driven Directional Ellipsoidal-MINUCB}
\label{alg:epoch-ell-minucb}
\small
\begin{algorithmic}[1]
\Require $D_{\mathrm{off}}$, $T$, prespecified
\(
1{=}\tau_1{<}\cdots{<}\tau_{K_{\mathrm{ep}}}{\le}T
\)
and
\(
\tau_{K_{\mathrm{ep}}+1}{=}T{+}1
\);
confidence schedules
$\{\delta_{t,\ell}^{\mathrm{on}}\}$,
$\{\delta_{t,\ell}^{\mathrm{pool}}\}$ as in Algorithm~\ref{alg:ellipsoidal-minucb},
$\delta_{\mathrm{off}}$, and
$\{\delta_k\}_{k=1}^{K_{\mathrm{ep}}}$ for $\beta_k^{\mathrm{on}}(\delta_k)$, satisfying
\(
\delta_{\mathrm{off}}{+}\sum_{k=1}^{K_{\mathrm{ep}}}\delta_k{\le}\delta_{\mathrm{bias}}
\)
\State Build
$G_{\mathrm{off}}$, $\hat\theta_{\mathrm{off}}$, and
$\beta_{\mathrm{off}}(\delta_{\mathrm{off}})$ from~$D_{\mathrm{off}}$
\State
$L\gets\lceil\log_2 T\rceil$; initialize
$\cH_{1,\ell}\gets\varnothing$ for $\ell=0,1,\dots,L$
\For{each epoch $k=1,2,\dots,K_{\mathrm{ep}}$}
    \State
    $G_k^{\mathrm{on}}\gets I{+}\sum_{s<\tau_k} x_{s,a_s}x_{s,a_s}^{\top}$;\quad
    $\hat\theta_k^{\mathrm{on}}\gets (G_k^{\mathrm{on}})^{-1}\sum_{s<\tau_k} x_{s,a_s}r_s$
    \State
    $\widehat M_{k,\mathrm{bias}}\gets\bigl((G_k^{\mathrm{on}})^{-1}{+}G_{\mathrm{off}}^{-1}\bigr)^{-1}$;\quad
    $\widehat\Delta_k\gets\hat\theta_k^{\mathrm{on}}-\hat\theta_{\mathrm{off}}$
    \State
    $\widehat\rho_k\gets
    \|\widehat\Delta_k\|_{\widehat M_{k,\mathrm{bias}}}
    {+}\beta_k^{\mathrm{on}}(\delta_k)
    {+}\beta_{\mathrm{off}}(\delta_{\mathrm{off}})$
    \For{each $t=\tau_k,\dots,\tau_{k+1}-1$}
        \State Execute the per-round layer loop of
        Algorithm~\ref{alg:ellipsoidal-minucb} at time~$t$,
        using \eqref{eq:epoch-pool-rad}--\eqref{eq:epoch-psi} in place of
        $\mathrm{rad}^{\mathrm{pool}}_{t,\ell}(\cdot)$ when forming the pooled UCB/LCB
    \EndFor
\EndFor
\end{algorithmic}
\end{algorithm}

\subsection{A Geometric Identity for the Parallel Sum}

We begin with the geometric identity underlying the learned certificate.

\begin{proposition}[Geometric characterization of the parallel sum]
\label{prop:parallel-sum-variational}
For any epoch~$k$ and any vector $x\in\R^d$,
\[
\|x\|_{\widehat M_{k,\mathrm{bias}}}^2
=
\inf_{x=u+v}
\Bigl\{
\|u\|_{G_k^{\mathrm{on}}}^2+\|v\|_{G_{\mathrm{off}}}^2
\Bigr\}.
\]
\end{proposition}

\begin{proof}
Write $A:=G_k^{\mathrm{on}}$ and $B:=G_{\mathrm{off}}$ for brevity, and let
$M:=(A^{-1}+B^{-1})^{-1}$.
For any decomposition $x=u+v$,
\[
\|u\|_A^2+\|v\|_B^2
=
u^\top A u + (x-u)^\top B(x-u).
\]
This is a strictly convex quadratic in~$u$, whose unique minimizer satisfies
\[
(A+B)u=Bx,
\qquad\text{hence}\qquad
u=(A+B)^{-1}Bx,
\quad
v=(A+B)^{-1}Ax.
\]
Substituting these expressions gives
\[
\inf_{x=u+v}
\Bigl\{
\|u\|_A^2+\|v\|_B^2
\Bigr\}
=
x^\top B(A+B)^{-1}A x.
\]
It remains to identify the matrix on the right with~$M$. Since
\[
A^{-1}+B^{-1}=A^{-1}(A+B)B^{-1},
\]
we have
\[
M
=
(A^{-1}+B^{-1})^{-1}
=
B(A+B)^{-1}A.
\]
Therefore the minimum equals
\[
x^\top M x
=
\|x\|_M^2
=
\|x\|_{\widehat M_{k,\mathrm{bias}}}^2,
\]
which proves the claim.
\end{proof}

\subsection{Validity of the Learned Directional Certificate}

We next verify that the learned pair
\((\widehat M_{k,\mathrm{bias}},\widehat\rho_k)\) defines a valid epoch-wise directional certificate.

\begin{theorem}[Validity of the data-driven directional bias certificate]
\label{thm:epoch-parallel-bound-valid}
Suppose the standard offline ridge and epoch-wise online events hold, i.e.\
\(
\|\theta^\dagger-\hat\theta_{\mathrm{off}}\|_{G_{\mathrm{off}}}\le\beta_{\mathrm{off}}(\delta_{\mathrm{off}})
\)
and, for every~$k$,
\(
\|\theta_*-\hat\theta_k^{\mathrm{on}}\|_{G_k^{\mathrm{on}}}\le\beta_k^{\mathrm{on}}(\delta_k)
\).
If
\(
\delta_{\mathrm{off}}+\sum_{k=1}^{K_{\mathrm{ep}}}\delta_k
\le\delta_{\mathrm{bias}}
\)
then, with probability at least $1-\delta_{\mathrm{bias}}$,
\[
\|\theta_*-\theta^\dagger\|_{\widehat M_{k,\mathrm{bias}}}
\le
\widehat\rho_k
\qquad
\forall k\in\{1,\ldots,K_{\mathrm{ep}}\}.
\]
\end{theorem}

\begin{proof}
With probability at least $1-\delta_{\mathrm{off}}$,
\(
\|\theta^\dagger-\hat\theta_{\mathrm{off}}\|_{G_{\mathrm{off}}}\le\beta_{\mathrm{off}}(\delta_{\mathrm{off}})
\).
For each epoch~$k$, the estimator
$\hat\theta_k^{\mathrm{on}}$ uses only data strictly before~$\tau_k$; the self-normalized ridge
inequality at time~$\tau_k$ yields
\(
\|\theta_*-\hat\theta_k^{\mathrm{on}}\|_{G_k^{\mathrm{on}}}\le\beta_k^{\mathrm{on}}(\delta_k)
\)
on an event of probability at least~$1-\delta_k$ (in the predictable epoch case, the same
time-uniform bound or a union over~$t$ and evaluation at~$\tau_k$ applies). A union bound gives
simultaneous validity with probability at least $1-\delta_{\mathrm{bias}}$ when
\(
\delta_{\mathrm{off}}+\sum_k\delta_k\le\delta_{\mathrm{bias}}
\).

On this event, for fixed~$k$, the parallel-sum definition implies
\(
\widehat M_{k,\mathrm{bias}}\preceq G_k^{\mathrm{on}}
\)
and
\(
\widehat M_{k,\mathrm{bias}}\preceq G_{\mathrm{off}}
\)
because
\(
(G_k^{\mathrm{on}})^{-1}{+}G_{\mathrm{off}}^{-1}\succeq (G_k^{\mathrm{on}})^{-1}
\)
and the symmetric inequality for~$G_{\mathrm{off}}$. Hence, for all~$v$,
\(
\|v\|_{\widehat M_{k,\mathrm{bias}}}\le\|v\|_{G_k^{\mathrm{on}}}
\)
and
\(
\|v\|_{\widehat M_{k,\mathrm{bias}}}\le\|v\|_{G_{\mathrm{off}}}
\).
Decompose
\(
\theta_*-\theta^\dagger
=
\widehat\Delta_k
+(\theta_*-\hat\theta_k^{\mathrm{on}})
-(\theta^\dagger-\hat\theta_{\mathrm{off}})
\)
with
\(
\widehat\Delta_k:=\hat\theta_k^{\mathrm{on}}-\hat\theta_{\mathrm{off}}
\).
The triangle inequality in the~$\widehat M_{k,\mathrm{bias}}$ norm, together with the confidence events and
Loewner domination, gives
\[
\|\theta_*-\theta^\dagger\|_{\widehat M_{k,\mathrm{bias}}}
\le
\|\widehat\Delta_k\|_{\widehat M_{k,\mathrm{bias}}}
+\beta_k^{\mathrm{on}}(\delta_k)
+\beta_{\mathrm{off}}(\delta_{\mathrm{off}})
=\widehat\rho_k.
\]
Since~$k$ is arbitrary, the claim holds for all~$k$.
\end{proof}

\subsection{Proof of Theorem~\ref{thm:epoch-parallel-regret}}
\label{app:k4-proof-epoch-parallel-regret}

\begin{proof}[Proof of Theorem~\ref{thm:epoch-parallel-regret}]
The online branch is unchanged, so the ordinary fallback regret term
$C_{\mathrm{SL}}\sqrt{Td\,L_T^3}$ and its proof are unchanged. For the pooled branch, on the event
of Theorem~\ref{thm:epoch-parallel-bound-valid},
\[
\|\theta_*-\theta^\dagger\|_{\widehat M_{k(t),\mathrm{bias}}}
\le
\widehat\rho_{k(t)}
\qquad\forall t.
\]
The epoch start times are part of the prespecified (or past-measurable) rule, and
$(\widehat M_{k,\mathrm{bias}},\widehat\rho_k)$ uses only data before~$\tau_k$; therefore the pair used in round~$t$
is $\cF_{t-1}$-measurable and the policy remains non-anticipatory, so the martingale and
self-normalized concentration events from the proof of Theorem~\ref{thm:main} are unchanged.
On the directional bias certificate event, for every round-$t$ queried direction~$x$,
\[
|x^\top(\theta_*-\theta^\dagger)|
\le
\widehat\rho_{k(t)}\|x\|_{\widehat M_{k(t),\mathrm{bias}}^{-1}}.
\]
In the pooled confidence step (Lemma~\ref{lem:pooled-valid} in the main appendix proof), the fixed
bias width
\(
\rho\,\|M_{\mathrm{bias}}^{-1/2}G_{\mathrm{off}}(A_{t,\ell}{+}G_{\mathrm{off}})^{-1}x\|_2
\)
is therefore replaced by
$\widehat\rho_{k(t)}\bigl\|\widehat M_{k(t),\mathrm{bias}}^{-1/2}G_{\mathrm{off}}(A_{t,\ell}{+}G_{\mathrm{off}})^{-1}x\bigr\|_2$,
i.e., $\widehat\rho_{k(t)}\widehat\psi_{k(t),t,\ell}(x)$. The UCB/LCB validity, elimination, and
stopping-layer arguments from the proof of Theorem~\ref{thm:main} then go through unchanged except
that the routed bias sum is~$\widetilde\Psi_T^{\mathrm{ep}}$, while the statistical (non-bias) part
of the pooled width still produces the
\(
8\sum_\ell
(\beta_{\mathrm{off}}{+}\gamma_\ell^{\mathrm{fin}})\sqrt{2|Q_\ell|\,\Lambda_\ell}
\)
term by the same pooled elliptical potential calculation. Intersecting the directional bias certificate event with
the original $1-T^{-2}$-probability event for the online, offline, and pooled self-normalized
events yields the stated probability and regret bound.
\end{proof}

\subsection{Doubling-Epoch Refinement}

We now specialize to a standard doubling schedule in order to obtain a closed-form refinement.

A default is
$K_{\mathrm{ep}}:=\lfloor\log_2 T\rfloor+1$,
\(
\tau_k:=2^{k-1}
\),
\(
\tau_{K_{\mathrm{ep}}+1}:=T+1
\).
To combine Theorems~\ref{thm:epoch-parallel-bound-valid}--\ref{thm:epoch-parallel-regret} with
closed-form radii, split
\(
\delta_{\mathrm{off}}:=\delta_{\mathrm{bias}}/2
\)
and
\(
\delta_k:=\delta_{\mathrm{bias}}/(2K_{\mathrm{ep}})
\)
and set
\[
\beta_k^{\mathrm{on}}
:=
\sigma\sqrt{
d\log(1{+}\tau_k)
+2\log(2K_{\mathrm{ep}}/\delta_{\mathrm{bias}})}
+S,
\qquad
\beta_{\mathrm{off}}^{\mathrm{bias}}
:=
\sigma\sqrt{
2\log\bigl(\det(G_{\mathrm{off}})^{1/2}/(\delta_{\mathrm{bias}}/2)\bigr)
}+S.
\]
In each epoch, with $\widehat\Delta_k$ and~$\widehat M_{k,\mathrm{bias}}$ built from~$D_{\mathrm{off}}$ and
online data before~$\tau_k$ (as in
Algorithm~\ref{alg:epoch-ell-minucb}), set
\(
\widehat\rho_k
:=
\|\widehat\Delta_k\|_{\widehat M_{k,\mathrm{bias}}}
+\beta_k^{\mathrm{on}}
+\beta_{\mathrm{off}}^{\mathrm{bias}}.
\)

\begin{theorem}[High probility bound for the data-driven directional bias certificate]
\label{thm:epoch-doubling-regret}
Under the doubling schedule, $\delta$-split, and
$(\widehat\rho_k,\widehat M_{k,\mathrm{bias}})$ from the previous paragraph, define
$n_k:=\tau_{k+1}-\tau_k$,
\[
\widehat c_k
:=
\lambda_{\max}\bigl(
G_{\mathrm{off}}^{1/2}\widehat M_{k,\mathrm{bias}}^{-1}G_{\mathrm{off}}^{1/2}
\bigr),
\qquad
\mathcal S_{\mathrm{ep}}
:=
\sum_{k=1}^{K_{\mathrm{ep}}} n_k\,\widehat\rho_k^2\,\widehat c_k.
\]
With probability at least $1-T^{-2}-\delta_{\mathrm{bias}}$, the regret satisfies
\[
R_T
\le
\min\left\{
C_{\mathrm{SL}}\sqrt{Td\,L_T^3},
\;
8\Gamma_{\mathrm{pool}}\sqrt{2T\,\widetilde\Lambda_T}
+8\sqrt{2\,\widetilde\Lambda_T\,\mathcal S_{\mathrm{ep}}}
\right\},
\]
and also the same bound with
$8\Gamma_{\mathrm{pool}}\sqrt{2T\,\widetilde\Lambda_T}$ replaced by the full layerwise statistical term
\(
8\sum_{\ell}
(\beta_{\mathrm{off}}+\gamma_\ell^{\mathrm{fin}})\sqrt{2|Q_\ell|\,\Lambda_\ell}
\),
where~$\Gamma_{\mathrm{pool}}$ is as in Theorem~\ref{thm:main}.
\end{theorem}

\begin{proof}[Proof of Theorem~\ref{thm:epoch-doubling-regret}]
This step specializes the epoch-wise bound to the doubling schedule introduced above.
Start from Theorem~\ref{thm:epoch-parallel-regret}. For $t$ in epoch~$k$, write
$A:=A_{t,\ell_t}$ and $G:=G_{\mathrm{off}}$. By definition of~$\widehat\psi$,
\[
\widehat\psi_{k,t,\ell_t}(x_{t,a_t})^2
=
x_{t,a_t}^{\top}
(A+G)^{-1}
G\widehat M_{k,\mathrm{bias}}^{-1}G
(A+G)^{-1}
x_{t,a_t}.
\]
The choice
\(
\widehat c_k := \lambda_{\max}(G^{1/2}\widehat M_{k,\mathrm{bias}}^{-1}G^{1/2})
\)
is equivalent to
\(
G\widehat M_{k,\mathrm{bias}}^{-1}G\preceq \widehat c_k\,G
\)
in Loewner order, hence
\(
\widehat\psi_{k,t,\ell_t}(x_{t,a_t})^2
\le
\widehat c_k
\,x_{t,a_t}^{\top}
(A+G)^{-1}G(A+G)^{-1}x_{t,a_t}
\).
Since $G\preceq A+G$,
\(
(A+G)^{-1}G(A+G)^{-1}\preceq (A+G)^{-1}
\)
and
\[
\widehat\psi_{k,t,\ell_t}(x_{t,a_t})
\le
\sqrt{\widehat c_k}
\,\|x_{t,a_t}\|_{(A_{t,\ell_t}+G_{\mathrm{off}})^{-1}}.
\]
Therefore
\[
\widetilde\Psi_T^{\mathrm{ep}}
\le
\sum_{t=1}^T
\widehat\rho_{k(t)}\sqrt{\widehat c_{k(t)}}
\,\|x_{t,a_t}\|_{(A_{t,\ell_t}+G_{\mathrm{off}})^{-1}}.
\]
By Cauchy--Schwarz,
\[
\widetilde\Psi_T^{\mathrm{ep}}
\le
\Bigl(\sum_{t=1}^T \widehat\rho_{k(t)}^2\widehat c_{k(t)}\Bigr)^{1/2}
\Bigl(\sum_{t=1}^T \|x_{t,a_t}\|_{(A_{t,\ell_t}+G_{\mathrm{off}})^{-1}}^2\Bigr)^{1/2}.
\]
The first factor equals~$\sqrt{\mathcal S_{\mathrm{ep}}}$ because $(\widehat\rho_k,\widehat c_k)$ are
constant on each epoch. The second is at most~$\sqrt{2\,\widetilde\Lambda_T}$ by
Lemma~\ref{lem:pooled-potential} and the definition
\(
\widetilde\Lambda_T=\sum_\ell\Lambda_\ell
\)
in Theorem~\ref{thm:main}. Thus
\(
\widetilde\Psi_T^{\mathrm{ep}}\le \sqrt{2\,\widetilde\Lambda_T\,\mathcal S_{\mathrm{ep}}}
\).
Substitute into the minimum expression of Theorem~\ref{thm:epoch-parallel-regret} and use the
Cauchy--Schwarz step
\(
\sum_\ell
(\beta_{\mathrm{off}}+\gamma_\ell^{\mathrm{fin}})\sqrt{2|Q_\ell|\,\Lambda_\ell}
\le
\Gamma_{\mathrm{pool}}\sqrt{2T\,\widetilde\Lambda_T}
\)
as in the proof of Theorem~\ref{thm:main} to obtain the two displayed bounds.
\end{proof}

\begin{proof}[Proof of Corollary~\ref{cor:epoch-adaptivity-penalty}]
Theorem~\ref{thm:epoch-doubling-regret} gives
\[
R_T
\le
\min\left\{
C_{\mathrm{SL}}\sqrt{Td\,L_T^3},
\;
8\Gamma_{\mathrm{pool}}\sqrt{2T\,\widetilde\Lambda_T}
+8\sqrt{2\,\widetilde\Lambda_T\,\mathcal S_{\mathrm{ep}}}
\right\}
\]
with probability at least $1-T^{-2}-\delta_{\mathrm{bias}}$.
Factor the pooled term as
\[
8\Gamma_{\mathrm{pool}}\sqrt{2T\,\widetilde\Lambda_T}
+8\sqrt{2\,\widetilde\Lambda_T\,\mathcal S_{\mathrm{ep}}}
=
8\Bigl(\Gamma_{\mathrm{pool}}+\sqrt{\mathcal S_{\mathrm{ep}}/T}\Bigr)
\sqrt{2T\,\widetilde\Lambda_T},
\]
which is exactly the displayed bound in the corollary.
\end{proof}

\section{Additional Experiments}
\label{app:more-experiments}

All experiments use $K{=}5$ arms, $d{=}5$ features, horizon $T{=}20000$, sub-Gaussian reward
noise~$\sigma{=}0.1$, and $n_{\mathrm{off}}{=}2000$ offline points unless stated otherwise.
Contexts are independent Gaussian draws in $\mathbb{R}^d$ for each arm and round, normalized to the
Euclidean unit ball. We set ridge $\lambda_{\mathrm{reg}}{=}1$, confidence budget
$\delta{=}0.01$, norm bound $S{=}1$, and offline covariate diversity multiplier
$\mathrm{div}{=}1.0$. Reported curves are means with variability bands over
$n_{\mathrm{runs}}{=}50$ trials with base seed~$42$.
The figures are multi-panel versions of the same protocol, with one panel per setting. All experiments run on a single CPU core.

\paragraph{Randomized bias composition.}
The online parameter is drawn with independent coordinates in~$[0.25,0.55]$.
The offline mean is
\(
\theta^\dagger = s\,\theta_* + u,
\)
where~$s\in(0,1]$ and $u$ is uniform on~$[0,c]^d$ with $c\in\{0,0.2,0.4,0.6,0.8\}$.
The pairs~$(s,c)$ are chosen so that as~$c$ increases, the additive offset~$u$ has larger range while the
scale factor~$s$ toward~$\theta_*$ decreases; at $c=0$ one has $s=1$ and $u=0$ (a pure re-scaling
parameterization of the online vector).
Figure~\ref{fig:app-exp-shift} stacks these five misspecification strengths; we plot Ellipsoidal-MINUCB
and the best purely offline policy on the same axis for each panel.
\begin{figure}[ht]
  \centering
  \includegraphics[width=0.98\linewidth]{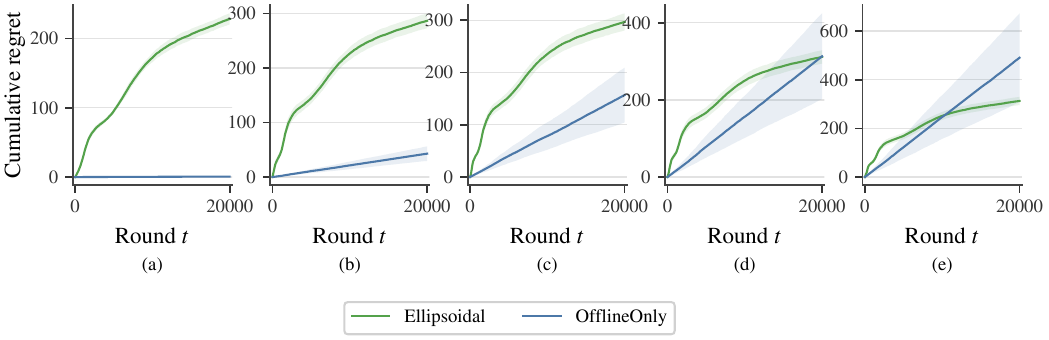}
  \caption{Randomized bias sweep: for each~$c$, mean cumulative regret under the~$(s,u)$ law above
  (left to right, increasing~$c$; five panels in the published figure).}
  \label{fig:app-exp-shift}
\end{figure}

\paragraph{Fixed online parameter, varying offline mean on one coordinate.}
We fix~$\theta_*=(2,1,1,1,1)$ and change only~$\theta^\dagger_1$ among
$\{0,1,2,3,4\}$ (with the remaining coordinates~$1$ in each case).
Figure~\ref{fig:app-exp-ablation} compares the full method to two stripped variants: one that only uses
the offline design matrix in the online branch, and one that only uses the online ridge design (no
offline pooling in the same sense as the main text).
\begin{figure}[ht]
  \centering
  \includegraphics[width=0.98\linewidth]{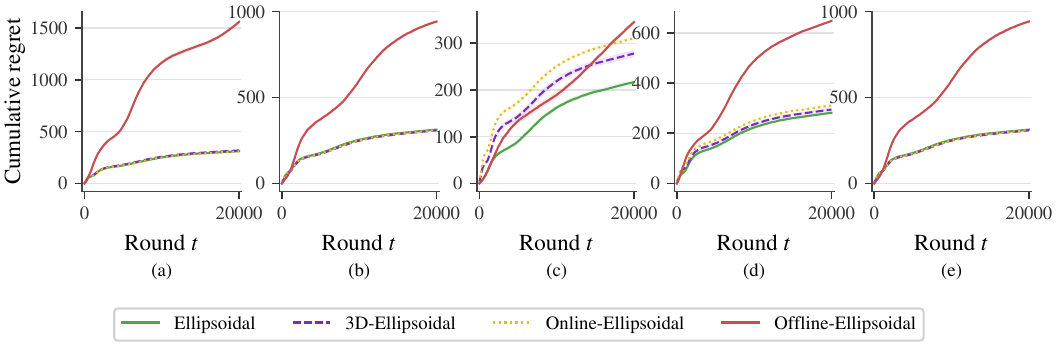}
  \caption{Ablation: five values of~$\theta^\dagger_1$ with~$\theta_*$ as above, same hyperparameters
  otherwise.}
  \label{fig:app-exp-ablation}
\end{figure}

\paragraph{Covariate shift on offline data only.}
The online and offline parameters are generated from the same random law, but offline covariates are
drawn from an axis-aligned box~$\prod_j[\ell_j,h_j]$ inside the support used at play time.
In particular, the first feature uses~$[0,1]$, and coordinates~$2,\ldots,d$ use~$[0,w]$ with~$w$ in a finite
decreasing set that includes $1$, a moderate value, and small values down to~$0.01$.
For small~$w$ the offline design places almost no energy on several directions, so the pooled statistics
in those directions are driven almost entirely by the online data even when~$\theta^\dagger=\theta_*$ on
average.
Figure~\ref{fig:app-exp-div} reports the same algorithm comparison as in Figure~\ref{fig:app-exp-shift}, here indexed by~$w$.
\begin{figure}[ht]
  \centering
  \includegraphics[width=0.98\linewidth]{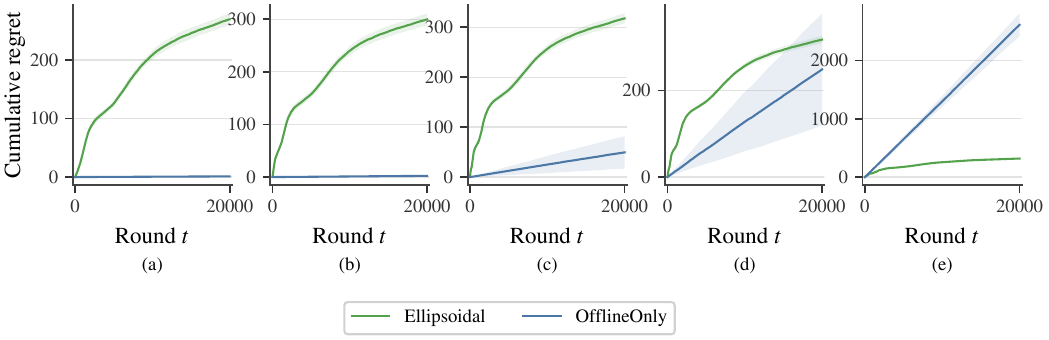}
  \caption{Offline-coverage study: each panel is a value of~$w$ in the family above (narrow
  offline support on several coordinates, wide online support; five panels in the published figure).}
  \label{fig:app-exp-div}
\end{figure}

\section{Broader Impact}
  \label{app: Broader Impact}
  This work studies how to reuse offline data in linear contextual bandits when the offline and online
  environments may differ. A potential positive impact is that safer reuse of historical data can
  reduce online exploration, which may lower deployment cost and reduce exposure to poor actions in
  applications such as recommendation, advertising, and pricing. A key risk is that if offline bias is
  underestimated, the method may over-trust historical data and propagate existing biases or spurious
  correlations. Our framework partly mitigates this through conservative confidence bounds and an
  explicit fallback to the online branch, but practical use in high-stakes settings still requires
  careful auditing and monitoring.

\end{document}